\documentclass{article}

\PassOptionsToPackage{numbers}{natbib}
\usepackage[preprint]{nips_2018}

\usepackage[utf8]{inputenc} 
\usepackage[T1]{fontenc}    
\usepackage{hyperref}       
\usepackage{url}            
\usepackage{booktabs}       
\usepackage{amsfonts}       
\usepackage{nicefrac}       
\usepackage{microtype}      
\usepackage{bbm}
\usepackage{amsmath}
\usepackage{latexsym}
\usepackage[linesnumbered,ruled]{algorithm2e}
\usepackage{graphicx}
\usepackage{tabu}
\usepackage{bm}
\usepackage{graphicx}

\usepackage[utf8]{inputenc}
\usepackage[english]{babel}
\usepackage{color}

\usepackage{float}
\usepackage{amsmath}
\usepackage{mathrsfs}
\usepackage{multirow}
\usepackage{amsmath,bm}
\usepackage{bbm}
\usepackage{wasysym}
\usepackage{amssymb}
\usepackage{lipsum} 
\usepackage{enumitem}
\setlist{nosep} 

\usepackage{times}
\usepackage{hhline}
\usepackage{empheq}

\newtheorem{fact}{Fact}[section]

\newtheorem{theorem}{Theorem}[section]

\newtheorem{example}{Example}[section]

\newtheorem{lemma}[theorem]{Lemma}
\newtheorem{proposition}[theorem]{Proposition}

\newenvironment{proof}{{\bfseries Proof:}}{$\Box$}

\newcommand{\Znote}[1]{{}}

\newcommand{\E}{\mathbbm{E}}
\newcommand{\nkt}{\mathrm{NKT}}

\newcommand{\sign}{\mathrm{sign}}

\newcommand{\cald}{\mathcal{D}}
\newcommand{\cale}{\mathcal{E}}
\newcommand{\calf}{\mathcal{F}}

\newcommand{\calq}{\mathcal{Q}}

\newcommand{\kt}{\mathrm{KT}}

\newcommand{\I}{\mathrm{I}}
\newcommand{\bfC}{\mathbf{C}}
\newcommand{\bfN}{\mathbf{N}}
\newcommand{\bR}{{\mathbb R}}

\newcommand{\calgo}{$\mathrm{Global}$-$\mathrm{kNN}$\xspace}
\newcommand{\wrongalgo}{$\mathrm{KT}$-$\mathrm{kNN}$\xspace}

\newcommand{\deltan}{\Delta^{(i)}_{1,2}}
\newcommand{\deltax}{\Delta^{(x)}_{1,2}}
\newcommand{\deltaxp}{\Delta^{(x)}_{1',2'}}
\newcommand{\deltaz}{\Delta^{(0)}_{1,2}}

\providecommand{\ie}{\emph{i.e.,} }
\providecommand{\eg}{\emph{e.g.,} }

\providecommand{\mypara}[1]{\smallskip\noindent\emph{#1} }
\providecommand{\myparab}[1]{\smallskip\noindent\textbf{#1} }

\DeclareMathOperator*{\argmin}{argmin}

\title{Towards Non-Parametric Learning to Rank}

\author{
  Ao Liu\\
  Department of Computer Science\\
  Rensselaer Polytechnic Institute\\
  Troy, NY 12180 \\
  \texttt{liua6@rpi.edu}\\
  \And
  Qiong Wu\\
  Department of Computer Science\\
  College of William and Mary\\
  Williamsburg, VA 23187 \\
  \texttt{qwu05@email.wm.edu}\\
  \AND
  Zhenming Liu\\
  Department of Computer Science\\
  College of William and Mary\\
  Williamsburg, VA 23187 \\
  \texttt{zliu@cs.wm.edu}\\
  \And
  Lirong Xia\\
  Department of Computer Science\\
  Rensselaer Polytechnic Institute\\
  Troy, NY 12180 \\
  \texttt{xial@cs.rpi.edu}\\
}
\begin{document}

\maketitle

\begin{abstract}
This paper studies a stylized, yet natural, learning-to-rank problem and points out the critical incorrectness of a widely used nearest neighbor algorithm. We consider a model with $n$ agents (users) $\{x_i\}_{i \in [n]}$ and $m$ alternatives (items) 
$\{y_j\}_{j \in [m]}$, each of which is associated with a latent feature vector. Agents rank items nondeterministically according to the Plackett-Luce model, where the higher the utility of an item to the agent, the more likely this item will be ranked high by the agent. 
Our goal is to find neighbors of an arbitrary agent or alternative in the latent space. 

We first show that the Kendall-tau distance based kNN produces incorrect results in our model. Next, we fix the problem by introducing a new algorithm with features constructed from ``global information'' of the data matrix. Our approach is in sharp contrast to most existing feature engineering methods. Finally, we design another new algorithm identifying similar alternatives. The construction of alternative features can be done using ``local information,'' highlighting the algorithmic difference between finding similar agents and similar alternatives. 
\end{abstract}

\section{Introduction}
This paper studies a stylized, yet natural, learning-to-rank problem and points out the  critical incorrectness of a widely used nearest neighbor algorithm. In this problem, let $X = \{x_1, \dots, x_n\}$ be the set of agents (users) and $Y = \{y_1, \dots, y_m\}$ be the set of alternatives. Each agent or alternative is associated with a latent feature vector $\mathbb{R}^d$, where $d \geq 1$. The utility of $y_j$ to $x_i$ is determined by $u(x_i, y_j)$, where $u(\cdot, \cdot)$ is a bivariate function. We observe a (partial) ranking $R_i$ of agent $i$ (for all $i$) over the alternatives. The distribution of ranking $R_i$ is determined by the alternatives' utilities, $\{u(x_i, y_j)\}_{j \in [m]}$. When $u(x_i, y_j)$ is larger, $y_j$ is more likely to rank higher in $R_i$. 

\myparab{The nearest neighbor problem.} For an $x_i$ and a parameter $\epsilon$, we aim to design an efficient algorithm that finds (almost) all $x_j$'s such that $\|x_j - x_i\|_2 \leq \epsilon$. The nearest neighbor problem for alternatives can also be defined similarly. 

This fundamental machine learning problem is embedded in many critical operations. 
For example, recommender systems use partial ranking information (partial observations of $R_i$’s) to estimate agents’ preferences over unranked alternatives, product designers estimate the demand curve of a new product based on consumers' past choices~\citep{Berry95:Automobile}, security firms estimate terrorists' preferences based on their past behavior, and political firms estimate campaign options based on voters' preferences~\cite{Liu:2009:LRI}.

\myparab{A widely-used algorithm produces incorrect results.} The most widely studied and deployed algorithm~\cite{Liu:2009:LRI,SamuelsS17} uses Kendall-Tau (KT) distance (see Section~\ref{sec:prelim}) as the metrics and uses k-nearest neighbors (kNN) to identify similar agents: for any given $x_i$, it finds all $x_j$ such that the KT distances between {$R_i$ and $R_j$} are minimized. We will refer to this algorithm as KT-kNN.

In this paper, we show that 
under a natural and widely applied preference model, 
the KT distance-based kNN for agents is \textbf{provably incorrect} even when the sample size grows to infinite. 



\myparab{Novel (correct) algorithms.} First, we design a new algorithm that correctly identifies similar agents based on $\{R_i\}_{i \in [n]}$. We introduce a set of new features, denoted by $F(x_i)$, so that $|F(x_i) - F(x_j)|_1$ enables us to identify similar agents. A salient property of $F(x_i)$ is that it relies on the rankings of other agents, which we will refer to as “global information”. This property is in sharp contrast to most existing practices of feature engineering in learning-to-rank algorithms~\cite{Liu:2009:LRI}.

Second, we design another new algorithm for identifying similar alternatives. We find that construction of alternative features can be done using local information, making identifying similar alternatives significantly easier.  

\mypara{Agent-wise or alternative-wise similarities.} Finding similar alternatives (items) is easier than finding similar agents (users) in collaborative filtering~\cite{sarwar2001item}: in practice, recommender systems based on “item-similarities” are usually more effective. One explanation is the “missing data problem”. Because there are often more users than items, the intersection between items ranked by two arbitrary users is often small, and measuring item similarities is usually unreliable.

Our result provides a new explanation of the performance discrepancies: under the Plackett-Luce model, finding similar agents is fundamentally more difficult than finding similar alternatives. 

\myparab{Additional remarks}

\mypara{Finding neighbors implies learning to rank.}  We focus on the problem of identifying nearest neighbors in this work. Our approach can be naturally extended to infer $x_i$’s  preferences over unranked alternatives by aggregating rankings from the neighbors using methods developed in the literature, such as in~\cite{Conitzer06:Kemeny,Alon06:Ranking,Ailon07:Aggregation,Kenyon07:How}.

\mypara{Nondeterministic preferences.} We assume that agent $x_i$ rank alternative $j$ according to her perceived untility $u(x_i,y_j)+\epsilon_{ij}$, where $u(\cdot, \cdot)$ is a radial basis function (RBF)~\cite{scholkopf2001learning} (\ie the value of $u(x, y)$ depends only on $\|x -y\|_2$) and $\epsilon_{ij}$ is a random noise. 

\mypara{Practical implications.} We focus on conceptual and theoretical investigation of the learning-to-rank problem with nondeterministic preferences. Although we point out the harm from using \wrongalgo, it may not be the root cause of a practical system based on \wrongalgo. To diagnose a ranking algorithm (specifically whether our theoretical results are relevant), one shall first check whether our model is suitable for his/her datasets.  



\section{Preliminaries}\label{sec:prelim}
\myparab{Our model.}  Let $X = \{x_1, ..., x_n\}\subset \bR^d$ be the set of agents (or users) and $Y = \{y_1, ..., y_m\} \subset \bR^d$ be the set of alternatives (or items).

\mypara{Utility functions.} Agent $i$'s utility on alternative $j$ is determined by a utility function $u(x_i, y_j)$. Throughout this paper, we use $u(x_i, y_j) = e^{-\|x_i - y_j\|_2}$,
where $\|\cdot\|_2$ is the $\ell_2$ norm. Most results developed in this paper can be generalized to many radial-basis functions~\cite{scholkopf2001learning}.

\mypara{Observation and rankings.}
We observe the ranking $R_i$ of each user $i$ in the decreasing order of perceived utility of the alternatives $\forall j, u(x_i,y_j)+\epsilon_{ij}$. When $\epsilon_{ij}$ follows a Gumbel distribution, then the nondeterministic preferences model is also known as the \emph{Plackett-Luce model}~\cite{Plackett75:Analysis,Luce77:Choice}. Let $(j_1, \dots j_m)$ be a random permutation of $[m]$, and we have
\begin{equation}\label{eqn:pleq}
\Pr\left[y_{j_1}\succ_i \cdots\succ_i y_{j_m}\right] = \prod_{\ell=1}^m\dfrac{u(x_i,y_{j_l})}{\sum_{k=l}^m u(x_i,y_{j_k})}\footnote{$y \succ_i y'$ denotes $y$ is ranked above $y'$ by agent $i$.}.
\end{equation}
\mypara{Distributions of $x_i$ and $y_j$.} We further assume that $x_i$ and $y_j$ are i.i.d. generated from fixed but unknown distributions $\cald_X$ and $\cald_Y$, respectively. Let the cdf (respectively, pdf) of $\cald_X$ and $\cald_Y$ be $F_X(\cdot)$ and $F_Y(\cdot)$ (respectively, $f_X(\cdot)$ and $f_Y(\cdot)$).
For exposition purposes, we make simplifying assumptions that \emph{(i)} $d = 1$, and \emph{(ii)} $\cald_X$ and $\cald_Y$ are on $[0, 1]$ and ``near uniform'' (\ie $\frac{\sup f_X(x)}{\inf f_X(x)}$ and $\frac{\sup f_Y(y)}{\inf f_Y(y)}$ are bounded by a constant $c$).
There assumptions are widely used in latent space models and can be relaxed via more careful analysis, see~\cite{Abraham:2013} and references therein.

\myparab{Our problem.} Given an agent $x_i$, we say $S$ is an \emph{$(\alpha, \beta)$-nearest neighbor set} for $x_i$ if
\begin{itemize}
\item For all $x_j$ such that $|x_i - x_j| \le \alpha$, $x_j \in S$.
\item For all $x_j$ such that $|x_i - x_j| > \beta$, $x_j \notin S$.
\item For any $x_j$ such that $\alpha < |x_i - x_j| \leq \beta$, we do not require any performance guarantee (\ie whether $x_j \in S$).
\end{itemize}
Similarly, we can define $(\alpha, \beta)$-nearest neighbor set for alternatives. In other words, all $x_j$'s that are within $\alpha$ away from $x_i$ should be included in $S$, and all $x_j$'s that are more than $\beta$ away from $x_i$ should not be in $S$. Therefore, our goal is to design efficient algorithms to compute $(\alpha, \beta)$-nearest neighbor sets with high probability,  where $\alpha, \beta = o(1)$.

\mypara{Partial observations and forecasts.} 
All results presented in this paper can be generalized to the partial ranking scenario, where each $R_i$ only consists of a subset $O_i \subseteq [m]$ of linear size. Furthermore, a natural problem in this scenario is to infer an agent's preferences over unranked alternatives. We note that an $(\alpha, \beta)$-nearest neighbor set for $x_i$ can be used to infer its rankings over the entire $[m]$ via existing techniques~\cite{Conitzer06:Kemeny,Alon06:Ranking,Ailon07:Aggregation}. Therefore, our problem is strictly harder than the preference estimation problem.

\myparab{Comparison to the KS model by~\citet{SamuelsS17}.} 
In the KS model, agent $i$'s ranking is deterministic, \ie $y_{j_1} \succ_i y_{j_2}$ iff $u(x_i, y_{j_1}) > u(x_i, y_{j_2})$, whereas our model allows to ``add noise'' to the observations, which is a more standard practice in learning to rank.

\subsection{Kendal-tau distance and prior algorithms}
Let $R_1$ and $R_2$ be two rankings over $[m]$ and let $R_1(i)$ denote the rank of the $i$-th alternative. The \emph{Kendall-tau distance} is
\begin{equation}
\kt(R_1,R_2)= \sum\limits_{i\neq j \in [m]}\I\left([R_1(i)-R_1(j)][R_2(i)-R_2(j)]<0\right),
\end{equation}
where $I(\cdot)$ is an indicator function that sets to one if and only if its argument is true. The Normalized Kendall-tau distance between $R_1$ and $R_2$ is $
\nkt(R_1,R_2) =\frac{1}{\binom{m}{2}}\kt(R_1,R_2)$.

\mypara{Nearest-neighbor algorithms.} See Algorithm~\ref{alg:classical}. We shall refer to the algorithm as \wrongalgo. This algorithm uses KT-distance as the distance metrics and run a kNN algorithm on top of it.
\begin{algorithm}[htp]
   \caption{Kendall-tau distance based kNN (\wrongalgo)}
   \label{alg:classical}
   {$\%$ \emph{Section~\ref{sec:wronglower} shows this algorithm is incorrect.}}\\
   {{\bfseries Input:} Agent set $X$, Alternative set $Y$, every agent $x_i$'s ranking $R_i$ over $[m]$, and the number of neighbors $k$.}\\
   {sort $x_j \neq x_i$ according to $\kt(R_i, R_j)$}\\
   {\ie $\kt(R_i, R_{j_1}) \leq \dots\leq \kt(R_{i}, R_{j_{n - 1}})$}\\
   {Return $\{R_{j_1}, \dots R_{j_k}\}$}
\end{algorithm}



\section{Incorrectness of \wrongalgo under Nondeterministic Preferences}\label{sec:wronglower}
This section explains why \wrongalgo is incorrect under the Plackett-Luce model. 
Let $\bar R_i$ be the ground-truth ranking of agent $i$, \ie the $j$-th element in $\bar R_{i}$ is the $j$-th largest value of the set $\{u(x_{i}, y_j)\}$.

\mypara{Intuition behind \wrongalgo.} Previously, \wrongalgo was considered correct because of two intuitions: (1) if $x_i$ and $x_{j}$ are close, then $\bar R_i$ and $\bar R_{j}$ are also close, and (2) if $\bar R_i$ and $\bar R_{j}$ are close, their ``realizations'' $R_i$ and $R_{j}$ will also be close. Therefore, when $x_j$ minimizes $\kt(R_i, R_j)$ for large $n$ and $m$, it also minimizes $\kt(\bar R_i, \bar R_j)$.  

Intuition (1) is theoretically grounded (see~\cite{SamuelsS17}). The key problematic part is that for nondeterministic users, $\kt(\bar R_i, \bar R_j)$ and $\kt(R_i, R_j)$ \emph{do not} have a monotone relationship. That is, an increase in $\kt(\bar R_i, \bar R_j)$ does not necessarily imply an increase in $\kt(R_i, R_j)$, and vice versa.

\begin{example}\label{example:simplewrong} Let $y_1 = 0.4$, $y_2 = 0.7$, and $x_1 = 0.5$. Consider the following two optimization problems:
\begin{equation}\label{eqn:minbarkt}
x_2 = \argmin_{x_2} \{ \kt(\bar R_1, \bar R_2)\}.
\end{equation}
\begin{equation}\label{eqn:minkt}
x_2 = \argmin_{x_2} \{\E [\kt(R_1,  R_2)]\}.
\end{equation} We can see that the structures of these two optimization problems are very different. For (\ref{eqn:minbarkt}), the optimal solution set is $(-\infty, 0.55)$. 
But for (\ref{eqn:minkt}), the optimal solution set is $(-\infty, -0.4)$. The key difference is that $x_1$ itself would be an optimal solution to (\ref{eqn:minbarkt}), but it is not an optimal solution to (\ref{eqn:minkt}).
\end{example}

\myparab{Interpreting the result.} 
We need to solve (\ref{eqn:minbarkt}) to find nearest neighbors, but the objective of \wrongalgo is closer to (\ref{eqn:minkt}). Specifically, 
consider a scenario with only two alternatives $\{y_1,y_2\}$ but $n$ is sufficiently large. The above example shows that $\arg \min_{x_2} \{\E [\kt(R_1,  R_2)]\}$ is far away from $x_1$. Because $n$ is sufficiently large, we have $\arg \min_{x_2 \in [0,1]} \{\E [\kt( R_1,  R_2)]\} \approx \arg \min_{x_2 \in X} \{\E [\kt(R_1, R_2)]\}$. The right side of the approximation resembles the nearest-neighbor approximation (\ie $\min_{x_2 \in X} \{\kt( R_1,  R_2)\}$) because $\kt(R_1, R_2)$ converges to its expectation for large $m$. Therefore, \wrongalgo solves (\ref{eqn:minkt}) which is different from solving (\ref{eqn:minbarkt}). 

This observation can be used to build a more general negative theorem, which implies that \wrongalgo cannot output any $(o(1),o(1))$-nearest neighbor set with high probability, because with $\Theta(1)$ probability the output of \wrongalgo is $\Theta(1)$ away.

\begin{theorem}\label{thm:wrong}
Let $\cald_Y = \text{Uniform}([0,1])$  and $\cald_X$ be a near-uniform distribution on $[0, 1]$. Let $x_i \in [0, 1]$ be arbitrary agent provides ranking $R_i$ (a random variable conditioned on $x_i$). We have:
$$\Pr\left[|x_i-x^*| = \Theta(1)\right] \geq \frac{1}{1+c_X} = \Theta(1),$$
where $x^* = \argmin_{x\in[0,1]} \left\{\E_{Y} \left[\nkt(R_{i}, R_x)\right]\right\}$ and $c_X = \frac{\sup f_X(x)}{\inf f_X(x)}$.
\end{theorem}
\vspace{-2mm}
\begin{proof} 
Since $\cald_X$ is near-uniform on $[0,1]$, we know $\Pr\left[x_i\in(0,\frac{1}{4}]\cup[\frac{3}{4},1)\right] \leq \frac{1}{1+c_X}=  \Theta(1)$. Then, we prove the following two claims indicating $|x_i-x^*| = \Theta(1)$ for all $x_i\in(0,\frac{1}{4}]\cup[\frac{3}{4},1)$:
\begin{itemize}[noitemsep,topsep=0pt]
\setlength{\itemsep}{2pt}
\setlength{\parskip}{2pt}
\setlength{\parsep}{0pt}
\item When $x_{i} \in \left(0,\frac{1}{4}\right]$:
$x^* = \argmin_{x\in[0,1]} \left\{\E_{Y} \left[\nkt(R_{i}, R_x)\right]\right\} = 0.$
\item When $x_{i} \in \left[\frac{3}{4},1\right)$:
$x^* = \argmin_{x\in[0,1]} \left\{\E_{Y} \left[\nkt(R_{i}, R_x)\right]\right\} = 1.$
\end{itemize}
{\bf Those two claims above also indicate that \wrongalgo cannot output any $(o(1),o(1))$-nearest neighbor set with $\Theta(1)$ probability.} 
Here, we focus on the most difficult case in the first claim above to highlight our new analytical techniques (see Appendix~\ref{a:wrong} for the full proof). Specifically, below we show
$$x^* = \argmin_{x\leq x_{i}} \E_{Y} \left[\nkt(R_{i}, R_x)\right]\} = 0.$$
Because $\calf(x) \equiv  \E_Y[\nkt(R_{i}, R_x) \mid x_{i}]$ is a continuous function of $x$, we use $d\calf(x)/dx$ to characterize the minimal point $x$. Specifically, we shall show that $d \calf(x)/dx > 0$ for all $x \in (0, x_i]$, which means the function is minimized when $x = 0$.

We next analyze $d\calf(x)/dx$ in the following events respectively:
\begin{itemize}[noitemsep,topsep=0pt]
\setlength{\itemsep}{0pt}
\setlength{\parskip}{0pt}
\setlength{\parsep}{0pt}
\item $\cale_0$: when $y_2 \geq y_1 \geq x$.
\item $\cale_1$: when $y_1 \leq x$ and $y_2 \geq 2 x_{i}$.
\item $\cale_2$: when $y_1 \leq x$ and $x \leq y_2 \leq 2x_{i} - x$.
\item $\cale_3$: when $y_1 \leq x$ and $2x_{i} - x \leq y_2 \leq 2x_{i}$.
\end{itemize}
According to the case-by-case analysis shown in Lemma~\ref{lemma:wrong1} and noting that $0\leq x\leq x_{n+1} \leq \frac{1}{4}$,
{\footnotesize
\begin{equation}\nonumber
\E[\Phi]= \sum_{0 \leq i \leq 3}\E_{y_1 < y_2}\E[\Phi \mid \cale_i]\Pr[\cale_i \mid y_1 < y_2] \geq 0.006x \geq  0.
\end{equation}
}
Equality sign holds if and only if $x = 0$. Therefore, $\calf(x)$ is minimized at $x = 0$ for $x \in [0, x_{i}]$. Appendix~\ref{a:wrong} completes the proof for $x \in [x_i, 1]$ using similar techniques. 
\end{proof}

\section{Nearest-neighbor with global information}\label{sec:userupper}
We propose a novel (and correct) kNN algorithm based on a new set of features $F(x_i)$ for all $i \in [n]$ that can be used by nearest-neighbor algorithms. Each feature needs to use \emph{global information} of all the rankings.

\myparab{Features based on all-pair normalized Kendall-tau distance.} We associate each agent $i \in [n]$ with a feature $F(x_i) \equiv F_i = \left(F_{i,1},\cdots,F_{i,n}\right) \in \mathbbm{R}^n$, where $F_{i, j}$ is constructed as below:\\
First, we group $Y$ to pairs so that the $k$-th pair consists of $\{y_{2k - 1},y_{2k}\}$. We then let: 
{\footnotesize
$$S_k =  I\left[\left(R_i(y_{2k-1}) - R_i(y_{2k})\right)\cdot\left(R_j(y_{2k - 1}) - R_j(y_{2k}) \right)>0\right].$$
}
Our features are
\begin{equation}\label{eqn:enkt}
F_{i, j} =  \frac{2}{m} \sum_{k = 1}^{m/2} S_k
\end{equation}
It follows that $S_k$'s are independent and $\E[S_k] = \E\left[\nkt(R_i, R_j)\right]$ for all $k$.
Then we define the distance function between agent $i$ and agent $j$ as
\begin{equation}
D(x_i, x_j) = \sum_{k \neq i, j}|F_{i, k} - F_{j, k}|.
\end{equation}
\myparab{The new kNN (\calgo).} Let $\epsilon = o(1)$, (\ie $\left(\frac{\log n}{n}\right)$). Our algorithm, hereafter \calgo, returns the set $\bfC_{\text{\calgo}}(x_i) = \{x_j: D(x_i, x_j) \leq \epsilon\}$. See Algorithm~\ref{alg:global}.

\begin{algorithm}[htp]
   \caption{Globle-kNN}
   \label{alg:global}
     {\bfseries Input:} Agent set $X$, Alternative set $Y$, every agent $x_i$'s ranking $R_i$ over $[m]$, and a threshold  $\epsilon$.\\
     Return $\bfC_{\text{Global-kNN}}(x_i)= \{x_j: D(x_i, x_j) \leq \epsilon\}$\\
\end{algorithm}





\myparab{Global vs. local information.} Algorithm \wrongalgo\xspace uses only local information to construct features (\ie feature of $x_i$ depends only on $R_i$), whereas \calgo needs to use all ranking information $\{R_j\}_{j \in [n]}$ to construct a feature $F(x_i)$. We note that relying on local information is unlikely to be sufficient to construct high-quality features. Instead, we need to use a slower procedure that takes advantage of all of the local information available to construct $F(x_i)$. Earlier works in network analysis (see~\cite{li2017world} and references therein) developed similar techniques in classifying nodes.

\begin{theorem}\label{thm:user} Using the notations above, for all $D_X$, $D_Y$ that are near-uniform on $[0,c]$ and $\epsilon\in(0,1)$, let $\bfC_{\text{\calgo}}(x_i) = \{x_j: D(x_i, x_j) \leq \epsilon \}$. There exist positive constants $c_1$, $c_2$ and $c_3$ such that $\bfC_{\text{\calgo}}(x_i)$ is an $(c_1\epsilon,\,c_2\sqrt\epsilon)$-nearest neighbor set of $x_i$ with probability at least $1-2n^2 e^{-c_3m}$.
\end{theorem}
\begin{proof} We first show that there exist constants $c_1'$ and $c_2'$ such that \begin{equation}\label{equ:dcd}
c_1'|x_i-x_j|^2 \leq \mathbbm{E}_{R_i,R_j,Y,X\setminus\{x_i,x_j\}}D(x_i,x_j) \leq c_2'|x_i-x_j|.
\end{equation}
Here, we outline a proof for the upper bound of (\ref{equ:dcd}), which is applicable to any distributions on $[0,c]$ (see Appendix~\ref{a:usersimilarity} for the full proof and lower bound analysis, which uses the near-uniform conditions).

All the analysis below assumes conditioning in knowing $x_i$ and $x_j$ (\ie $\E[\cdot]$ means $\E[\cdot \mid x_i, x_j]$). W.l.o.g., assume $x_i \leq x_j$. We use techniques similar to those developed in Theorem \ref{thm:wrong}. Specifically let
$p_i = \mathbbm{P}\left[y_1\succ_i y_2 \mid  y_1, y_2, x_i\right]$. We have
{\footnotesize
\begin{equation}\nonumber
\begin{split}
\mathbbm{E}_{R_i,R_j}\left[D(x_i, x_j)\right] &= \mathbbm{E}_{R_i,R_j}\left[\mathbbm{E}_{x_k}\Big|\mathbbm{E}_{\,y_1,\,y_2}\;\left\{\left[p_k(1-p_i)+p_i(1-p_k)\right]\right. - \left.\left[p_k(1-p_j)+p_j(1-p_k)\right]\right\}\Big|\right]\\
&= \mathbbm{E}_{R_i,R_j}\left[\mathbbm{E}_{x_k}\Big|\mathbbm{E}_{\,y_1\leq\,y_2}\left[\left(p_i-p_j\right)(1-2p_k)\right]\Big|\right].
\end{split}
\end{equation}
}
Let us define three events:
\begin{itemize}[noitemsep,topsep=0pt]
\setlength{\itemsep}{0pt}
\setlength{\parskip}{0pt}
\setlength{\parsep}{0pt}
\item $\cale_1$: when $y_1,y_2\in[0,x_i]$ or $y_1,y_2\in[x_j,c]$.
\item $\cale_2$: when $y_1\in[0,x_i]$ and $y_2\in[x_j,c]$.
\item $\cale_3$: when $y_1\in[x_i,x_j]$ or $y_2\in[x_i,x_j]$.
\end{itemize}

We compute $(p_i-p_j)(1-2p_k)$ conditioned on the  three events (\ie $\left[\left(p_i-p_j\right)(1-2p_k)\mid \cale_{\ell}\right]$ for $\ell = 1,2,3$).

\mypara{Event $\cale_1$:} One can see that $
\mathbbm{E}\left[\left(p_i-p_j\right)(1-2p_k)\mid \cale_1\right] = 0$ by using a symmetric argument.\\
\mypara{Event $\cale_2$.}
We have $p_j \mid \cale_2 \geq  p_i-\frac{\epsilon}{2}+O(\epsilon^2).$\\
\mypara{Event $\cale_3$.} We have $\mathbbm{E}_{\,y_1\leq\,y_2}\left[\left(p_i-p_j\right)(1-2p_k)\mid \cale_3\right]\leq\frac{\epsilon}{2}+ O(\epsilon^2).$

$D(x_i,x_j)\leq c_2'\epsilon'$ follows from combining all results for $\{\cale_i\}_i$.

Next, we show the tail bound of $D(x_i,x_j)$, where $|D(x_i, x_j) - \E \left[D(x_i, x_j)\right]|$ decays exponentially in $m$. Observing that $S_k$'s are independent in any $F_{i,j}$, we have $\Pr\left[\,|D(x_i, x_j)-\mathbbm{E}\left[D(x_i, x_j)\right]| \geq \delta\mathbbm{E}\left[D(x_i, x_j)\right]\,\right] \leq n e^{-c_4(\delta) m}$ according to standard Chernoff bound. Combining the tail bound above with (\ref{equ:dcd}), we know there exists constants $c_1''$, $c_2''$ and $c_3$ such that
\begin{equation}\label{equ:cdcddc}
\Pr\left[\,c_1'' |x_i-x_j|^2 \leq D(x_i, x_j) \leq c_2'' |x_i-x_j|\,\right] = 1-ne^{-c_3m}.\end{equation}
We now interpret (\ref{equ:cdcddc}) in the context of  $(\alpha,\beta)$-nearest neighbor set. We analyze the $\beta$ part first. For any agent $x_j\in S_\beta = \left\{x_j: |x_i-x_j| > c_2\sqrt\epsilon = \sqrt{\epsilon/c_2''}\right\}$, we have,
\begin{equation}\nonumber
\begin{split}
\Pr\left[x_j\in\bfC_{\text{\calgo}}(x_i)\right] &= \Pr\left[D(x_i,x_j) > c_2\epsilon\right] \leq 1-\Pr\left[D(x_i,x_j) \leq c_2''|x_i-x_j| \right]\\
&\leq 1-\Pr\left[c_1''|x_i-x_j|^2 \leq D(x_i,x_j)\leq c_2''|x_i-x_j|\right] \leq ne^{-c_3m}.
\end{split}
\end{equation}
We also note that there are at most $n$ agents in $S_{\beta}$. By applying union bound to all agents in $S_{\beta}$, we got the conclusion of $\Pr\left[S_{\beta}\cap\bfC_{\text{\calgo}}(x_i) = \emptyset\right] \geq 1-n^2e^{-c_3m}$.

Letting $S_{\alpha} = \left\{x_j: |x_i-x_j| \leq c_1\epsilon = \epsilon/c_1''\right\}$, we get $\Pr\left[S_{\alpha}\cap\bfC_{\text{\calgo}}(x_i) = S_{\alpha}\right] \geq 1-n^2e^{-c_3m} $. Then, Theorem~\ref{thm:user} follows by applying union bound to $\alpha$'s and $\beta$'s conclusions. \end{proof}

\section{Nearest-Neighbor algorithm for alternatives}\label{sec:item}
This section designs an algorithm for finding $(\alpha, \beta)$-nearest neighbor set for an alternative $y_i$. While global information is needed for finding $(\alpha, \beta)$-nearest neighbor for agents, we need only local information for alternatives.
For exposition purpose, we focus on uniform $\cald_X$ and $\cald_Y$.

\myparab{Additional notations.} Define $d_x(y_i, y_j) = 2(I(R_x(y_i) - R_x(y_j) <  0) - 1)$. Here, $x \in [0, 1]$ represents an agent and $R_x$ represents its ranking over all alternatives. Intuitively, $d_x(y_i, y_j)$ is $1$ if $y_j\succ_x y_i  $ and is $-1$ otherwise.
Next, define $D_a(y_i, y_j) =\left|\frac 1 n \sum_{k \in [n]} (2I(R_k(y_i) - R_k(y_j) < 0) + 1)\right|$. Note that the terms in the summation are i.i.d. random variables, each of which has the same distribution with $d_a(y_i, y_j)$.

\myparab{Our algorithm and its intuition.} Our goal is to find an $(\alpha, \beta)$-nearest neighbor set  $\bfN(y_i)$ of $y_i$. To determine whether $y_i$ and $y_j$ are close, we shall check $D_a(y_i, y_j)$: when $y_i \approx y_j$, with probability exactly $\approx \frac 1 2$ that $R_k(y_i) < R_k(y_j)$, which implies
$\E[D_a(y_i, y_j)] \approx 0$. When $y_i$ and $y_j$ are  far away, then it is unlikely that $\E[d_x(y_i, y_j)] = 0$ (there is a catch; see below). As $D_a(y_i, y_j)$ is the mean of $n$ copies of independent $d_x(y_i, y_j)$, it will drift away from $0$.

\mypara{A ``bug'' due to symmetry.} One issue of the above argument is that
$|y_i - y_j|$ large does not always imply $\E d_x(y_i, y_j) \neq 0$. For any $y_i \neq \frac 1 2$,
when $y_j = 1 - y_i \neq y_i$, we have:
{\footnotesize
\begin{eqnarray*}
 \E[d_x(y_i, y_j)]
& = & \int_{0}^{\frac 1 2} d_{\frac 1 2 - \epsilon} (y_i, y_j) + d_{\frac 1 2 + \epsilon}(y_i, y_j) d \epsilon.
\end{eqnarray*}
}
One can check that $d_{\frac 1 2 - \epsilon} (y_i, y_j) + d_{\frac 1 2 + \epsilon}(y_i, y_j) = 0$ for any $\epsilon$. Therefore, $\E[d_x(y_i, y_j)] = 0$ for $y_j = 1 - y_i$.

\mypara{A two-step algorithm.} Let $\epsilon$ be a suitable parameter and $\ell = \Theta(1/\epsilon)$. We design a two-step algorithm to circumvent the symmetric bug:
\begin{itemize}[noitemsep,topsep=0pt]
\setlength{\itemsep}{0pt}
\setlength{\parskip}{0pt}
\setlength{\parsep}{0pt}
\item \emph{Step 1. Construction of candidate set:} We let $\bfC(y_i) = \{j: D_a(y_i, y_j) \leq 1/\ell\}$. All neighbors of $y_i$ are in $\bfC(y_i)$. 
\item \emph{Step 2. Filtering:} We design a procedure that can determine whether $y_j$ is close to $y_i$ or to $1 - y_i$ for all $y_j \in \bfC(y_i)$. Using this algorithm, and use the procedure filter out all the alternatives in $\bfC(y_i)$ that are not close to $y_i$. 
\end{itemize}
Details of step 1 and 2 will be given below. The performance of our algorithm is characterized by the following proposition.
\begin{proposition} Using the above notations, let $y_i$ be an arbitrary alternative. There exists an efficient algorithm that constructs an $(c_1 \epsilon, c_2 \sqrt{\epsilon})$-nearest neighbor set for any $\epsilon = \omega(\frac{\log n}{n})$. Here, $c_1$ and $c_2$ are two suitably chosen constants.
\end{proposition}

\noindent{\bf Step 1:  Construction of candidate set. }Let $\mu_x(y_i, y_j) = \E[d_x(y_i, y_j)]$.

\begin{lemma}\label{lem:itemsim} Let $\tau_i = |\frac 1 2 - y_i|$ (for all $i \in [m]$) and $\delta_{i, j} = |\tau_i - \tau_j|$. Then there exist constants $c_1$ and $c_2$ such that
\begin{equation}
c_1 \delta^2_{i, j} \leq \mu_x(y_i, y_j) \leq c_2 \delta_{i, j}.
\end{equation}
\end{lemma}




{\noindent\bf Step 2: Filtering out unwanted alternatives.} 
Now we have a candidate set $\bfC(y_i)$ such that for any $y_j \in \bfC(y_i)$, $y_j$ is either close to $y_i$ or $1 - y_i$. Next, we describe an algorithm that
eliminates elements in $\bfC(y_i)$ that's not close to $y_i$. We now formally describe the problem.

\myparab{The Split-cluster problem.} Let $\bfC(y_i)$ be a set such that for any $y_j \in \bfC(y_i)$, either $|y_j - y_i| \leq \delta$ or $|y_j - (1- y_i)| \leq \delta$, where $\delta = o(1)$. Our goal is to find all $y_j \in \bfC(y_i)$ such that $|y_j - y_i| \leq 5\delta$.

Our split-cluster algorithm is shown in Algorithm~\ref{alg:split}, with analysis and remarks shown in Appendix.

\begin{lemma}\label{lem:cluster} When $n = \Omega(1/\delta^2)$, Algorithm~\ref{alg:split} returns all $y_j \in \bfC$ such that $|y_i - y_j| \leq 5 \delta$.
\end{lemma}

\begin{algorithm}[tb]
   \caption{Split-cluster}
   \label{alg:split}
   {\bfseries Input:} a set $\bfC(y_i)$ such that for any $y_j \in \bfC(y_i)$, $|y_j - y_i| < \epsilon$ or $|y_j - (1-y_i)| < \epsilon$.\\
   Let $\calq = \{S(y_i, y_j): y_j \in \bfC(y_i)\}$\\
   $\calq_1, \calq_2 = \mbox{k-means}(\calq)$ ($k = 2$).
   \% \emph{Comment: assume $\calq_2$'s centroid is smaller}.\\
   Return $\bfC^+= \{y_j: S(y_i, y_j) \in \calq_1\}$.
\end{algorithm}

\section{Numeric validation}
This section presents results of experiments based on synthetic data to validate our theoretical results. We randomly generated 1200 agents and 6000 alternatives according to $\cald_X = \cald_Y = \text{Uniform}\left(\left[0,5\right]\right)$. Then we introduce a new agent $x_{n + 1}$ and reveal its partial ranking to the system. Our goal is to predict $\mathbbm{P}\left[y_i\geq_{n+1}y_j\right]$. We examine three algorithms: \emph{(i)} \wrongalgo, \emph{(ii)} \calgo, \emph{(iii)} Ground-truth (\ie directly using the nearest neighbors of an agent $x_{n+1}$ in latent space). The ground-truth algorithm cannot be implemented in practice and only serves as a optimal bound for any kNN based algorithms. We consider $k \in [20, 500]$ ($k$ is the number of neighbors to keep). See Figure~\ref{fig:Exp}. One can see that \wrongalgo consistently has bad performance whereas \calgo's performance is very close to the lower bound. 

Figure~\ref{fig:Exp}(c) shows experiments for high-dimensional latent spaces ($d\in[1,10]$) under the same setting as 1D except $\cald_X = \cald_Y = \text{Uniform}\left(\left[0,5\big/{\sqrt d}\right]^d\right)$. We see \wrongalgo consistently has worse performance than \calgo, whose performance is very close to ground truth.
\begin{figure}
\centering
\includegraphics[width = 1\textwidth]{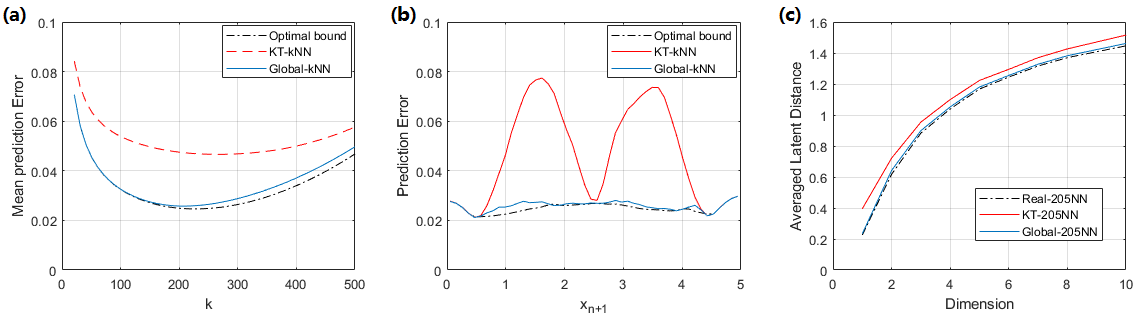}
\caption{\footnotesize {\textbf{(a)} Averaged prediction error for different $k \in [20, 500]$. The optimal errors for \wrongalgo, \calgo, and optimal bounds are 0.0466, 0.0258, and 0.0246 respectively. \textbf{(b)} Prediction errors when $k = 205$ (the best $k$). The performance discrepancies between \wrongalgo and \calgo become more pronounced for large $k$, especially when $x_{n + 1}$ is close to 1.25 or 3.75. \textbf{(c)} Averaged distance from $x_{n+1}$ to 205-nearest neighbors in $d$ dimensional latent spaces. }}
\label{fig:Exp}
\end{figure}

\section{Additional related work}
\myparab{Nonparametric learning in practice.}
Our model is sometimes considered as a non-parametric model. 
Nonparametric preference learning methods are widely applied in practice but little is known about their theoretical guarantees. 
Our work is related to the recent line of work in preference completion~\cite{McNee:2006,Liu:2008:ERA:1390334.1390351,Cremonesi:2010,Wang:2012:AVS:2396761,Wang:2014:VNF:2648782.2542048,huang2015listwise,Cheng:2017,SamuelsS17}. Some most recent algorithms (\eg~\cite{Wang:2014:VNF:2648782.2542048,huang2015listwise}) have impressive performance in practice, but have no theoretical explanations justifying the successes.

\myparab{Non-ranking observations.} There is a rich literature (\eg see~\cite{Herlocker1999:An-algorithmic,Liu2008:EigenRank,Bobadilla2013:Recommender,Lee2016:Blind} and references therein) 
on learning information about $\{u(x_i y_j)\}_{i,j}$ based on partial observations. For example, in the classical collaborative filtering problem, noisy observations of $u(x_i, y_j)$ (\eg the observation is $u(x_i, y_j)+ \epsilon$ for some white noise $\epsilon$). 
These results are not comparable to ours. 
Other work~\cite{Kleinberg2003:Convergent,Kleinberg2008:Using} assumes an observation model  related to ours: an alternative $j$ is more likely to be used/evaluated by an agent $i$ if $u(x_i, y_j)$ is high. 

\myparab{Low-rank assumption.} The work \cite{Park2015:Preferences,Gunasekar2016:Preference} assumes the matrix $[u(x_i, y_j)]_{i,j}$ (or its expectation) is low ranks. This matrix is full rank in all the utility functions and models considered in our program. Furthermore, their loss functions are not in terms of rank correlations (the most natural choice). 

\myparab{Parametric inference.}
Parametric preference learning has been extensively studied in machine learning, especially learning to rank~\cite{Cheng10:Label,Mollica16:Bayesian,Negahban17:Rank,Azari12:Random,Azari14:Computing,Azari13:Preference,Azari13:Generalized,Maystre2015:Fast,Khetan16:Data,Hughes15:Computing,Zhao16:Learning}. These method often assume the existence of a parametric model, usually Random Utility Model or Mallows' model.


\section{Concluding remarks}
This paper introduced a natural learning-to-rank model, and showed that under this model a widely-used KT-distance based kNN algorithm failed to find similar agents (users). To fix the problem, we introduced a new set of features for agents that relies on the ranking of other agents (\ie relying on ``global information''). We also design an algorithm for finding similar alternatives, based on using only local information. The two algorithmic results showed that the ``item-similarity'' problem is fundamentally different from the ``user-similarity'' problem. 

\myparab{Generalization.} We made two assumptions in our analysis: \emph{(i)} we observe each agent's full ranking over $[m]$; and \emph{(ii)} $\cald_X$ and $\cald_Y$ are in 1-dimensional space. Relaxing assumption (i) is straightforward because we need only develop specialized tail bounds for $|D(x_i, x_j) - \E D(x_i, x_j)|$ (discussed in Section~\ref{sec:userupper}). Relaxing assumption (ii), however, is challenging because our analysis heavily relies on symmetric properties over the 1-dimensional lines, many of which break in high-dimensional space. We note that in practice, the improvement of predictive power using high-dimensional models is usually incremental~\cite{li2017world}.

\myparab{Limitation.} RBF utilities are not universally applicable in all recommender systems (\eg in some circumstances, ``cosine similarities'' are more suitable utility functions). This paper's major contribution is the theoretical investigation of a fundamental learning-to-rank problem. 
It remains a future work to apply our results to understand their impacts on practical recommender systems.



\bibliographystyle{plainnat}

{\footnotesize \begin{thebibliography}{39}
\providecommand{\natexlab}[1]{#1}
\providecommand{\url}[1]{\texttt{#1}}
\expandafter\ifx\csname urlstyle\endcsname\relax
  \providecommand{\doi}[1]{doi: #1}\else
  \providecommand{\doi}{doi: \begingroup \urlstyle{rm}\Url}\fi

\bibitem[Abraham et~al.(2013)Abraham, Chechik, Kempe, and
  Slivkins]{Abraham:2013}
Ittai Abraham, Shiri Chechik, David Kempe, and Aleksandrs Slivkins.
\newblock Low-distortion inference of latent similarities from a multiplex
  social network.
\newblock In \emph{Proceedings of the Twenty-fourth Annual ACM-SIAM Symposium
  on Discrete Algorithms}, SODA '13, pages 1853--1883, Philadelphia, PA, USA,
  2013. Society for Industrial and Applied Mathematics.

\bibitem[Ailon(2007)]{Ailon07:Aggregation}
Nir Ailon.
\newblock Aggregation of partial rankings, p-ratings and top-m lists.
\newblock In \emph{Proceedings of the Annual ACM-SIAM Symposium on Discrete
  Algorithms (SODA)}, 2007.

\bibitem[Alon(2006)]{Alon06:Ranking}
Noga Alon.
\newblock Ranking tournaments.
\newblock \emph{SIAM Journal of Discrete Mathematics}, 20:\penalty0 137--142,
  2006.

\bibitem[Azari~Soufiani et~al.(2012)Azari~Soufiani, Parkes, and
  Xia]{Azari12:Random}
Hossein Azari~Soufiani, David~C. Parkes, and Lirong Xia.
\newblock Random utility theory for social choice.
\newblock In \emph{Proceedings of Advances in Neural Information Processing
  Systems (NIPS)}, pages 126--134, Lake Tahoe, NV, USA, 2012.

\bibitem[Azari~Soufiani et~al.(2013{\natexlab{a}})Azari~Soufiani, Chen, Parkes,
  and Xia]{Azari13:Generalized}
Hossein Azari~Soufiani, William Chen, David~C. Parkes, and Lirong Xia.
\newblock Generalized method-of-moments for rank aggregation.
\newblock In \emph{Proceedings of Advances in Neural Information Processing
  Systems (NIPS)}, Lake Tahoe, NV, USA, 2013{\natexlab{a}}.

\bibitem[Azari~Soufiani et~al.(2013{\natexlab{b}})Azari~Soufiani, Parkes, and
  Xia]{Azari13:Preference}
Hossein Azari~Soufiani, David~C. Parkes, and Lirong Xia.
\newblock {Preference Elicitation For General Random Utility Models}.
\newblock In \emph{Proceedings of Uncertainty in Artificial Intelligence
  (UAI)}, Bellevue, Washington, USA, 2013{\natexlab{b}}.

\bibitem[Azari~Soufiani et~al.(2014)Azari~Soufiani, Parkes, and
  Xia]{Azari14:Computing}
Hossein Azari~Soufiani, David~C. Parkes, and Lirong Xia.
\newblock {Computing Parametric Ranking Models via Rank-Breaking}.
\newblock In \emph{Proceedings of the 31st International Conference on Machine
  Learning}, Beijing, China, 2014.

\bibitem[Berry et~al.(1995)Berry, Levinsohn, and Pakes]{Berry95:Automobile}
Steven Berry, James Levinsohn, and Ariel Pakes.
\newblock Automobile prices in market equilibrium.
\newblock \emph{Econometrica}, 63\penalty0 (4):\penalty0 841--890, 1995.

\bibitem[Bobadilla et~al.(2013)Bobadilla, Ortega, Hernando, and
  Guti{\'e}Rrez]{Bobadilla2013:Recommender}
J.~Bobadilla, F.~Ortega, A.~Hernando, and A.~Guti{\'e}Rrez.
\newblock Recommender systems survey.
\newblock \emph{Knowledge-Based Systems}, 46:\penalty0 109--132, 2013.

\bibitem[Cheng et~al.(2017)Cheng, Wang, Ma, Sun, and Xiong]{Cheng:2017}
Peizhe Cheng, Shuaiqiang Wang, Jun Ma, Jiankai Sun, and Hui Xiong.
\newblock Learning to recommend accurate and diverse items.
\newblock In \emph{Proceedings of the 26th International Conference on World
  Wide Web}, WWW '17, pages 183--192, 2017.
\newblock ISBN 978-1-4503-4913-0.

\bibitem[Cheng et~al.(2010)Cheng, Dembczynski, and
  H{\"u}llermeier]{Cheng10:Label}
Weiwei Cheng, Krzysztof~J. Dembczynski, and Eyke H{\"u}llermeier.
\newblock Label ranking methods based on the plackett-luce model.
\newblock \emph{Proceedings of the 27th International Conference on Machine
  Learning (ICML-10)}, pages 215--222, 2010.

\bibitem[Conitzer et~al.(2006)Conitzer, Davenport, and
  Kalagnanam]{Conitzer06:Kemeny}
Vincent Conitzer, Andrew Davenport, and Jayant Kalagnanam.
\newblock Improved bounds for computing {K}emeny rankings.
\newblock In \emph{Proceedings of the National Conference on Artificial
  Intelligence (AAAI)}, pages 620--626, Boston, MA, USA, 2006.

\bibitem[Cremonesi et~al.(2010)Cremonesi, Koren, and Turrin]{Cremonesi:2010}
Paolo Cremonesi, Yehuda Koren, and Roberto Turrin.
\newblock Performance of recommender algorithms on top-n recommendation tasks.
\newblock In \emph{Proceedings of the Fourth ACM Conference on Recommender
  Systems}, RecSys '10, pages 39--46. ACM, 2010.
\newblock ISBN 978-1-60558-906-0.

\bibitem[Gunasekar et~al.(2016)Gunasekar, Koyejo, and
  Ghosh]{Gunasekar2016:Preference}
Suriya Gunasekar, Oluwasanmi~O. Koyejo, and Joydeep Ghosh.
\newblock {Preference Completion from Partial Rankings}.
\newblock In \emph{Advances in Neural Information Processing Systems}, 2016.

\bibitem[Herlocker et~al.(1999)Herlocker, Konstan, Borchers, and
  Riedl]{Herlocker1999:An-algorithmic}
Jonathan~L. Herlocker, Joseph~A. Konstan, Al~Borchers, and John Riedl.
\newblock An algorithmic framework for performing collaborative filtering.
\newblock In \emph{Proceedings of the 22nd annual international ACM SIGIR
  conference on Research and development in information retrieval}, pages
  230--237, 1999.

\bibitem[Huang et~al.(2015)Huang, Wang, Liu, Ma, Chen, and
  Veijalainen]{huang2015listwise}
Shanshan Huang, Shuaiqiang Wang, Tie-Yan Liu, Jun Ma, Zhumin Chen, and Jari
  Veijalainen.
\newblock Listwise collaborative filtering.
\newblock In \emph{Proceedings of the 38th International ACM SIGIR Conference
  on Research and Development in Information Retrieval}, pages 343--352. ACM,
  2015.

\bibitem[Hughes et~al.(2015)Hughes, Hwang, and Xia]{Hughes15:Computing}
David Hughes, Kevin Hwang, and Lirong Xia.
\newblock {Computing Optimal Bayesian Decisions for Rank Aggregation via MCMC
  Sampling}.
\newblock In \emph{Proceedings of the Conference on Uncertainly in Artificial
  Intelligence (UAI)}, pages 385--394, 2015.

\bibitem[Katz{-}Samuels and Scott(2017)]{SamuelsS17}
Julian Katz{-}Samuels and Clayton Scott.
\newblock Nonparametric preference completion.
\newblock \emph{CoRR}, abs/1705.08621, 2017.
\newblock URL \url{http://arxiv.org/abs/1705.08621}.

\bibitem[Kenyon-Mathieu and Schudy(2007)]{Kenyon07:How}
Claire Kenyon-Mathieu and Warren Schudy.
\newblock {How to Rank with Few Errors: A PTAS for Weighted Feedback Arc Set on
  Tournaments}.
\newblock In \emph{Proceedings of the Thirty-ninth Annual ACM Symposium on
  Theory of Computing}, pages 95--103, San Diego, California, USA, 2007.

\bibitem[Khetan and Oh(2016)]{Khetan16:Data}
Ashish Khetan and Sewoong Oh.
\newblock Data-driven rank breaking for efficient rank aggregation.
\newblock In \emph{Proceedings of the 33rd International Conference on Machine
  Learning}, volume~48, 2016.

\bibitem[Kleinberg and Sandler(2003)]{Kleinberg2003:Convergent}
Jon Kleinberg and Mark Sandler.
\newblock Convergent algorithms for collaborative filtering.
\newblock In \emph{Proceedings of the 4th ACM conference on Electronic
  commerce}, pages 1--10, 2003.

\bibitem[Kleinberg and Sandler(2008)]{Kleinberg2008:Using}
Jon Kleinberg and Mark Sandler.
\newblock Using mixture models for collaborative filtering.
\newblock \emph{Journal of Computer and System Sciences}, 74\penalty0
  (1):\penalty0 49--69, 2008.

\bibitem[Lee et~al.(2016)Lee, Li, Shah, and Song]{Lee2016:Blind}
Christina~E. Lee, Yihua Li, Devavrat Shah, and Dogyoon Song.
\newblock {Blind Regression: Nonparametric Regression for Latent Variable
  Models via Collaborative Filtering}.
\newblock In \emph{Advances in Neural Information Processing Systems}, 2016.

\bibitem[Li et~al.(2017)Li, Wong, Liu, and Kanade]{li2017world}
Cheng Li, Felix~MF Wong, Zhenming Liu, and Varun Kanade.
\newblock From which world is your graph.
\newblock In \emph{Advances in Neural Information Processing Systems}, pages
  1468--1478, 2017.

\bibitem[Liu and Yang(2008{\natexlab{a}})]{Liu2008:EigenRank}
Nathan~N. Liu and Qiang Yang.
\newblock {EigenRank: a ranking-oriented approach to collaborative filtering}.
\newblock In \emph{Proceedings of the 31st annual international ACM SIGIR
  conference on Research and development in information retrieval}, pages
  83--90, 2008{\natexlab{a}}.

\bibitem[Liu and Yang(2008{\natexlab{b}})]{Liu:2008:ERA:1390334.1390351}
Nathan~N. Liu and Qiang Yang.
\newblock Eigenrank: A ranking-oriented approach to collaborative filtering.
\newblock In \emph{Proceedings of the 31st Annual International ACM SIGIR
  Conference on Research and Development in Information Retrieval}, SIGIR '08,
  pages 83--90, 2008{\natexlab{b}}.
\newblock ISBN 978-1-60558-164-4.

\bibitem[Liu(2009)]{Liu:2009:LRI}
Tie-Yan Liu.
\newblock Learning to rank for information retrieval.
\newblock \emph{Found. Trends Inf. Retr.}, 3\penalty0 (3):\penalty0 225--331,
  March 2009.
\newblock ISSN 1554-0669.

\bibitem[Luce(1977)]{Luce77:Choice}
R.~Duncan Luce.
\newblock The choice axiom after twenty years.
\newblock \emph{Journal of Mathematical Psychology}, 15\penalty0 (3):\penalty0
  215--233, 1977.

\bibitem[Maystre and Grossglauser(2015)]{Maystre2015:Fast}
Lucas Maystre and Matthias Grossglauser.
\newblock {Fast and accurate inference of Plackett-Luce models}.
\newblock In \emph{Proceedings of the 28th International Conference on Neural
  Information Processing Systems}, pages 172--180, 2015.

\bibitem[McNee et~al.(2006)McNee, Riedl, and Konstan]{McNee:2006}
Sean~M. McNee, John Riedl, and Joseph~A. Konstan.
\newblock Being accurate is not enough: How accuracy metrics have hurt
  recommender systems.
\newblock In \emph{CHI '06 Extended Abstracts on Human Factors in Computing
  Systems}, CHI EA '06, pages 1097--1101, New York, NY, USA, 2006. ACM.
\newblock ISBN 1-59593-298-4.

\bibitem[Mollica and Tardella(2016)]{Mollica16:Bayesian}
Cristina Mollica and Luca Tardella.
\newblock Bayesian {P}lackett--{L}uce mixture models for partially ranked data.
\newblock \emph{Psychometrika}, pages 1--17, 2016.

\bibitem[Negahban et~al.(2017)Negahban, Oh, and Shah]{Negahban17:Rank}
Sahand Negahban, Sewoong Oh, and Devavrat Shah.
\newblock Rank centrality: Ranking from pairwise comparisons.
\newblock \emph{Operations Research}, 65\penalty0 (1):\penalty0 266--287, 2017.

\bibitem[Park et~al.(2015)Park, Neeman, Zhang, Sanghavi, and
  Dhillon]{Park2015:Preferences}
Dohyung Park, Joe Neeman, Jin Zhang, Sujay Sanghavi, and Inderjit~S. Dhillon.
\newblock {Preference Completion: Large-scale Collaborative Ranking from
  Pairwise Comparisons}.
\newblock In \emph{Proceedings of the 32nd International Conference on
  International Conference on Machine Learning}, pages 1907--1916, 2015.

\bibitem[Plackett(1975)]{Plackett75:Analysis}
Robin~L. Plackett.
\newblock The analysis of permutations.
\newblock \emph{Journal of the Royal Statistical Society. Series C (Applied
  Statistics)}, 24\penalty0 (2):\penalty0 193--202, 1975.

\bibitem[Sarwar et~al.(2001)Sarwar, Karypis, Konstan, and
  Riedl]{sarwar2001item}
Badrul Sarwar, George Karypis, Joseph Konstan, and John Riedl.
\newblock Item-based collaborative filtering recommendation algorithms.
\newblock In \emph{Proceedings of the 10th international conference on World
  Wide Web}, pages 285--295. ACM, 2001.

\bibitem[Scholkopf and Smola(2001)]{scholkopf2001learning}
Bernhard Scholkopf and Alexander~J Smola.
\newblock \emph{Learning with kernels: support vector machines, regularization,
  optimization, and beyond}.
\newblock MIT press, 2001.

\bibitem[Wang et~al.(2012)Wang, Sun, Gao, and Ma]{Wang:2012:AVS:2396761}
Shuaiqiang Wang, Jiankai Sun, Byron~J. Gao, and Jun Ma.
\newblock Adapting vector space model to ranking-based collaborative filtering.
\newblock In \emph{Proceedings of the 21st ACM International Conference on
  Information and Knowledge Management}, CIKM '12, pages 1487--1491, 2012.
\newblock ISBN 978-1-4503-1156-4.

\bibitem[Wang et~al.(2014)Wang, Sun, Gao, and
  Ma]{Wang:2014:VNF:2648782.2542048}
Shuaiqiang Wang, Jiankai Sun, Byron~J. Gao, and Jun Ma.
\newblock Vsrank: A novel framework for ranking-based collaborative filtering.
\newblock \emph{ACM Trans. Intell. Syst. Technol.}, 5\penalty0 (3):\penalty0
  51:1--51:24, July 2014.
\newblock ISSN 2157-6904.

\bibitem[Zhao et~al.(2016)Zhao, Piech, and Xia]{Zhao16:Learning}
Zhibing Zhao, Peter Piech, and Lirong Xia.
\newblock {Learning Mixtures of Plackett-Luce Models}.
\newblock In \emph{Proceedings of the 33rd International Conference on Machine
  Learning (ICML-16)}, 2016.

\end{thebibliography}
}

\newpage
\appendix
\section{Missing analysis for analyzing \wrongalgo}\label{a:wrong}
This section presents the missing analysis in Section~\ref{sec:wronglower}. We have the following three major lemmas. 
\begin{lemma}\label{lemma:wrong1}
Let $\cald_X$ be a uniform distribution on $[0, 1]$, $x_{i} \leq 0.25$, and $O_{i} = [m]$. Let $x \in [0,x_{i}]$ be an arbitrary agent and $R_x$ be the ranking of the agent $x$ (which is a random variable conditioned on $x$). We have:
$$x^* = \argmin_{x\geq 1-x_{n+1}} \E_{Y} \left[\nkt(R_{n +1}, R_x)\right]\} = 1-x_{n+1}.$$
\end{lemma}
\begin{proof}
Because $\calf(x) \equiv  \E_Y[\nkt(R_{i}, R_x) \mid x_{i}]$ is a continuous function of $x$, we use $d\calf(x)/dx$ to characterize the minimal point $x$. Specifically, we shall show that $d \calf(x)/dx > 0$ for all $x \in [0, 1]$, which means the function is minimized when $x = 0$.

We next calculate $d\calf(x)/dx$. Let $p_x(y_1, y_2) = \Pr[y_1 \succ_x y_2 \mid y_1, y_2, x]$ and $p_{i}(y_1, y_2) = \Pr[y_1 \succ_{i}y_2 \mid y_1, y_2, x_i]$. When the context is clear, we can write $p_x$ and $p_{i}$ instead. We have
{\footnotesize
\begin{eqnarray*}
\E[\nkt(R_x, R_{i}) \mid x_{i}]
&= \E_{y_1, y_2}\left[\E[\nkt(R_x, R_{i}) \mid y_1, y_2, x_{i}] \right]= \E_{y_1, y_2}[p_x(1-p_{i}) + p_{i}(1-p_x) \mid x_i].
\end{eqnarray*}
}
Therefore, we have
$\frac{d\calf(x)}{dx} = \E_{y_1, y_2}\left[(1-2p_{i})\frac{dp_x}{dx}\right].$
Let $\deltax = |y_2 - x| - |y_1 - x|$ and $\deltan = |y_2 - x_i| - |y_1 - x_i|$. We have
{\footnotesize
\begin{eqnarray*}
& & \E_{y_1, y_2}\left[(1-2p_{n + 1})\frac{dp_x}{dx} \mid x_i\right]
 =  \E_{y_1 < y_2}\left[\frac{e^{-\deltan} - 1}{e^{-\deltan} + 1}\cdot \frac{\sign(y_1 - x) - \sign(y_2 - x)}{4\cosh^2(\deltax/2) }\mid x_i\right].
\end{eqnarray*}
}
The last equation uses $y_1$ and $y_2$ are symmetric.

To simplify the notation, let $\Phi(y_1, y_2, x, x_i) = \frac{e^{-\deltan} - 1}{e^{-\deltan} + 1}\cdot \frac{\sign(y_1 - x) - \sign(y_2 - x)}{4\cosh^2(\deltax/2)}$ (sometimes we simply use $\Phi$ when the context is clear). Define the following events:
\begin{itemize}[noitemsep,topsep=0pt]
\setlength{\itemsep}{0pt}
\setlength{\parskip}{0pt}
\setlength{\parsep}{0pt}
\item $\cale_0$: when $y_2 \geq y_1 \geq x$.
\item $\cale_1$: when $y_1 \leq x$ and $y_2 \geq 2 x_{i}$.
\item $\cale_2$: when $y_1 \leq x$ and $x \leq y_2 \leq 2x_{i} - x$.
\item $\cale_3$: when $y_1 \leq x$ and $2x_{i} - x \leq y_2 \leq 2x_{i}$.
\end{itemize}
We have
\begin{equation}
\E_{y_1 \leq y_2}[\Phi \mid x_{i}] = \sum_{0 \leq j \leq 3}\E_{y_1 \leq y_2}[\Phi\mid \cale_j, x_i]\Pr[\cale_j, x_i].
\end{equation}
First note that $\E_{y_1 \leq y_2}[\Phi \mid \cale_0, x_i] = 0$ because
$\sign(y_1 - x) - \sign(y_2 - x)\mid \cale_0 = 0.$

Next, we  compute each of $\E_{y_1 \leq y_2}[\Phi \mid \cale_i, x_i]$ for $i = 1,2,3$. Before proceeding, we state useful properties of an important function:

\begin{fact}\label{fact:wrong}
Let $g(x) = \frac{e^{-x}-1}{e^{-x}}$, we have: \emph{(i)} $g(x) = - g(-x)$, \emph{(ii)} $g''(x) \cdot g(x) \geq 0$, and \emph{(iii)} $|g(x)| \leq \frac 1 2 g(x)$.
\end{fact}
Our analysis below assumes that the expectation is conditioned on $x_i$ (we hide it to simplify the notations). 

{Case 1. $\E[\Phi \mid \cale_1]$. } Observe that $\deltax \in [2x_{i}-2x, 1 - x]$, $\sign(y_1-x)-\sign(y_2-x) = -2$ and $\deltan > 0$. Therefore, $\E[\deltan \mid \cale_1] = \frac{1-x}{2} - x_{i}$. We have
{\footnotesize
$\E_{y_1 < y_2}[\Phi\mid \cale_1]
 \geq  \frac{2e(e-1)}{\left(1+e\right)^3}\cdot\left(\frac{1+x}{2}-x_{i}\right).
$
}
This deviation need to use the 2nd item in Fact \ref{fact:wrong} (\ie $g(\cdot)$ is concave). 

{Case 2. $\E[\Phi \mid \cale_2]$.} We have $\deltan < 0$, $\sign(y_1-x)-\sign(y_2-x) = -2$ and $\deltax \in [-x, 2x_{i} - x]$. Therefore, $\E[\deltan \mid \cale_2] = -\frac{x_{i}}{2}$.
With some algebraic manipulation, we have
\vspace{-1mm}
{\footnotesize
\begin{eqnarray*}
 \E[\Phi \mid \cale_2] 
& = & \E_{y_1 < y_2}\left[g(\deltan \frac{-2}{4\cosh^2(\deltax/2)}\right]
\geq \E_{y_1,y_2}\left[g\left(\deltan\right)\cdot\left.\frac{-2}{\left(e^{0}+e^{0}\right)^2}\right| \cale_2 \right] \\
& \geq^* &  -\frac{1}{2}\,\mathbbm{E}_{y_1,y_2}\left[\left.\frac{1}{2}{\deltan}\right| \cale_2 \right]
\geq  -\frac{x_{i}}{8}.
\end{eqnarray*}
}
$\geq^*$ comes from the (iii) in  Fact \ref{fact:wrong}.

\mypara{Case 3. $\E[\Phi \mid \cale_3, x_i]$.} We have $\deltax \in [x_i - 3x, 2x_{i}]$, $\sign(y_1-x)-\sign(y_2-x) = -2$ and $\E[\deltan\mid \cale_3] = 0$. Let $y'_1 = 2x_{i} - y_1$ and $y'_2 = 2x_{i} - y_2$. Note that we also have $y'_1 \leq x$ and $2x_{i} - x \leq y'_2 \leq 2 x_{i}$. Therefore,
{\footnotesize
\begin{eqnarray*}
\E[\Phi \mid \cale_3, x_i]
& = & \E_{y_1 < y_2}[\Phi(y_1, y_2, x, x_{i}) + \Phi(y'_1, y'_2, x, x_{i})\mid \cale_3 \deltan > 0] \\
& = & \E_{y_1 < y_2}\left[g(\deltan) \left(\frac{-2}{4\cosh^2(\deltax/2)} - \frac{2}{4\cosh^2(\deltaxp/2)}\right)\right]
\text{ (using (i) in Fact \ref{fact:wrong})}\\
& \geq & \E_{y_1 < y_2}\left[g(\deltan)\left(\frac{-2}{4\cosh^2(\deltan/2)} - \frac{2}{4\cosh^2(0)}\right)\right] \\
& \geq & \left(\frac{1}{4\cosh^2(x_{i}/2)} - \frac{1}{4\cosh^2(0)}\right)x_{n + 1}. \\
\end{eqnarray*}
}
Knowing that $0\leq x\leq x_{n+1} \leq \frac{1}{4}$, we have
{\footnotesize
\begin{equation}\nonumber
\begin{split}
\E[\Phi]= & \sum_{1 \leq i \leq 3}\E_{y_1 < y_2}\E[\Phi \mid \cale_i]\Pr[\cale_i \mid y_1 < y_2]\\
\geq & \frac{2e(e-1)}{(1+e)^3}\left(\frac{1+x}{2} - x_{i}\right)\cdot x(1-2x_{i}) - \frac{x_{i}}{8}x(2x_{i} - 2x)+ \left(\frac{1}{4\cosh^2(x_{i}/2)} - \frac{1}{4\cosh^2(0)}\right)x_{i}x^2 \\
\geq & x\left[\frac{2e(e-1)}{\left(1+e\right)^3}\cdot\frac{1}{4}\cdot\frac{1}{2}-\frac{1}{8}\cdot\frac{1}{4}\cdot\frac{1}{2}-0.0155\cdot\frac{1}{4}\cdot\frac{1}{4}\right] \geq  0.
\end{split}
\end{equation}
}
Equality sign holds if and only if $x = 0$. Therefore, $\calf(x)$ is minimized at $x = 0$ for $x \in [0, x_{i}]$. Appendix~\ref{a:wrong} completes the proof for $x \in [x_i, 1]$ using similar techniques. 
\end{proof}

\begin{lemma}\label{lemma:wrong2}
Let $\cald_X$ be a uniform distribution on $[0, 1]$, $x_{n + 1} \leq 0.25$, and $O_{n + 1} = [m]$. Let $x \in [1-x_{n+1},1]$ be an arbitrary agent and $R_x$ be the ranking of the agent $x$ (which is a random variable conditioned on $x$). We have:
$$x^* = \argmin_{x\geq 1-x_{n+1}} \E_{Y} \left[\nkt(R_{n +1}, R_x)\right]\} = 1-x_{n+1}.$$
\end{lemma}

\begin{lemma}\label{lemma:wrong3}
Let $\cald_X$ be a uniform distribution on $[0, 1]$, $x_{n + 1} \leq 0.25$, and $O_{n + 1} = [m]$. Let $x \in [x_{n+1},1-x_{n+1}]$ be an arbitrary agent and $R_x$ be the ranking of the agent $x$ (which is a random variable conditioned on $x$). We have:
$$x^* = \argmin_{x \in [x_{n+1},1-x_{n+1}]} \E_{Y} \left[\nkt(R_{n +1}, R_x)\right]\} = x_{n+1}.$$
\end{lemma}

Lemma~\ref{lemma:wrong2} can be proved using techniques similar to those presented in Theorem~\ref{lemma:wrong1}. We focus on Lemma~\ref{lemma:wrong3}. 

\begin{proof}[Proof of Lemma~\ref{lemma:wrong3}]
We have
$$\mathbbm{E}_{Y}\left[\nkt(R_{n +1}, R_x)\right] = \mathbbm{E}_{y_1,y_2}\left[p_{n+1}+(1-2p_{n+1})p_x\right].$$
and specially,
\begin{equation}\nonumber
\begin{split}
&\mathbbm{E}_{Y}\left[\nkt(R_{n +1}, R_x)\mid x = x_{n+1}\right]\\
=\;& \mathbbm{E}_{y_1,y_2}\left[2p_{n+1}(1-p_{n+1})\right].
\end{split}
\end{equation}
To simplify the notation, define the following events for Lemma \ref{lemma:wrong3}:
\begin{itemize}[noitemsep,topsep=0pt]
\setlength{\itemsep}{0pt}
\setlength{\parskip}{0pt}
\setlength{\parsep}{0pt}
\item $\cale_0$: when $y_1,y_2 \geq x+x_{n+1}$.
\item $\cale_1$: when $y_1,y_2 \leq x+x_{n+1}$.
\item $\cale_2$: when $y_1 \geq x+x_{n+1} \geq y_2$ or $y_2 \geq x+x_{n+1} \geq y_1$.
\end{itemize}
\mypara{Case 0. $\E[\nkt(R_{n +1}, R_x) \mid \cale_0]$. }

Using the same tool as used in Case 0 of Theorem \ref{thm:wrong} in article, we have,
\begin{equation}\label{eqn:case1}
\E[\nkt(R_{n +1}, R_x) \mid \cale_0] = 0
\end{equation}

\mypara{Case 1. $\E[\nkt(R_{n +1}, R_x) \mid \cale_1]$. }

For any $y_1,y_2 \in [0,x+x_{n+1}]$, we define $y_1' = x+x_{n+1}-y_1$ and $y_2' = x+x_{n+1}-y_2$. If we assume $y_1\geq y_2$, we always have $y_1'\leq y_2'$. By further defining $p_{n+1}' = \mathbbm{P}[y_1'\succ_{n+1} y_2']$ and $p_{x}' = \mathbbm{P}[y_1'\succ_{x} y_2']$. We have $p_{n+1}' = p_x$ and $p_{x}' = p_{n+1}$ and
\begin{equation}\label{equ:x}
\begin{split}
&\mathbbm{E}_{y_1,y_2}\left[\nkt(R_{n +1}, R_x)\Big|\cale_1\right]\\
=\;& \mathbbm{E}_{y_1,y_2}\left[p_{n+1}+(1-2p_{n+1})p_x\Big|\cale_1\right]\\
=\;& \mathbbm{E}_{y_1\leq y_2}\left[p_{n+1}+(1-2p_{n+1})p_x \right.\left.+ p_{n+1}'+(1-2p_{n+1}')p_x'\Big|\cale_1\right]\\
=\;& 2\mathbbm{E}_{y_1\leq y_2}\left[p_{n+1}+(1-2p_{n+1})p_x\Big|\cale_1\right].\\
\end{split}
\end{equation}
Similarly, we have,
\begin{equation}\label{equ:n}
\begin{split}
&\mathbbm{E}_{y_1,y_2}\left[\nkt(R_{n +1}, R_x)\Big|\cale_1,\,x = x_{n+1}\right]\\
=\;& \mathbbm{E}_{y_1, y_2}\left[2p_{n+1}(1-p_{n+1})\Big|\cale_1\right]\\
=\;& \mathbbm{E}_{y_1\leq y_2}\left[2p_{n+1}(1-p_{n+1})+ 2p_{n+1}'(1-p_{n+1}')\Big|\cale_1\right]\\
=\;& 2\mathbbm{E}_{y_1\leq y_2}\left[p_{n+1}(1-p_{n+1})+p_{x}(1-p_{x})\Big|\cale_1\right].\\
\end{split}
\end{equation}
Combining Equation \ref{equ:x} and \ref{equ:n}, we have:\\
\begin{equation}\label{eqn:case2}
\begin{split}
&\mathbbm{E}_{y_1,y_2}\left[\nkt(R_{n +1}, R_x)\Big|\cale_1\right]- \mathbbm{E}_{y_1,y_2}\left[\nkt(R_{n +1}, R_x)\Big|\cale_1,\,x = x_{n+1}\right]\\
=\;& 2\mathbbm{E}_{y_1\leq y_2}\left[p_{n+1}+(1-2p_{n+1})p_x\Big|\cale_1\right]- 2\mathbbm{E}_{y_1\leq y_2}\left[p_{n+1}(1-p_{n+1})+p_{x}(1-p_{x})\Big|\cale_1\right]\\
=\;& 2\mathbbm{E}_{y_1\leq y_2}\left[p_{n+1}^2-2p_{x}p_{n+1}+p_{x}^2\Big|\cale_1\right]\\
\geq\;&0.
\end{split}
\end{equation}

\mypara{Case 2. $\E[\nkt(R_{n +1}, R_x) \mid \cale_2]$. }

It is clear that $y_1 \geq x+x_{n+1} \geq y_2$ and $y_2 \geq x+x_{n+1} \geq y_1$ are equivalent because of the symmetry between $y_1$ and $y_2$. So, we only study the event $\cale_{2.1} = y_2 \geq x+x_{n+1} \geq y_1$, 
\begin{equation}\nonumber
\begin{split}
&\mathbbm{E}_{y_1,y_2}\left[\nkt(R_{n +1}, R_{x})\Big|\cale_{2.1}\right] - \mathbbm{E}_{y_1,y_2}\left[\nkt(R_{n +1}, R_{x})\Big|\cale_{2.1},\,x = x_{n+1}\right]\\
=\;&\mathbbm{E}_{y_1,y_2}\left[p_x(1-p_{n+1})+p_{n+1}(1-p_x)-2p_{n+1}(1-p_{n+1})\Big|\cale_{2.1}\right]\\
=\;&\mathbbm{E}_{y_1,y_2}\left[(p_x-p_{n+1})(1-2p_{n+1})\Big|\cale_{2.1}\right].\\
\end{split}
\end{equation}

Because $|y_1-x_{n+1}|\leq|y_2-x_{n+1}|$ when $y_1\leq x_i+x_{n+1}\leq y_2$, we always have $1-2p_{n+1} \leq 0$ . Then,
\begin{equation}\nonumber
\begin{split}
p_{n+1} =\;& \frac{e^{-|x_{n+1}-y_1|}}{e^{-|x_{n+1}-y_1|}+e^{-|y_2-x_{n+1}|}}
=\;\frac{e^{-\left(|x_{n+1}-y_1|-|x-x_{n+1}|\right)}}{e^{-\left(|x_{n+1}-y_1|-|x-x_{n+1}|\right)}+e^{-|y_2-x|}}
\geq\; \frac{e^{-|x-y_1|}}{e^{-|x-y_1|}+e^{-|y_2-x_i|}}
=\; p_x.
\end{split}
\end{equation}
Thus, we have
\begin{equation}\label{eqn:case3}
\begin{split}
&\mathbbm{E}_{y_1,y_2}\left[\nkt(R_{n +1}, R_{x})\Big|\cale_{2}\right] - \mathbbm{E}_{y_1,y_2}\left[\nkt(R_{n +1}, R_{x})\Big|\cale_{2},\,x = x_{n+1}\right]\geq\;0.
\end{split}
\end{equation}

And Lemma 3.4 follows by combining (\ref{eqn:case1}), (\ref{eqn:case2}) and (\ref{eqn:case3}).
\end{proof}
$x_{n+1}\leq 0.25$ part of Theorem 3.1 follows by adding  Lemma \ref{lemma:wrong2}- \ref{lemma:wrong3} and $x_{n+1}\geq 0.75$ part follows by symmetry.
\section{Missing proofs for \calgo}\label{a:usersimilarity}
\subsection{Upper bound}
All the analysis below assumes conditioning in knowing $x_i$ and $x_j$ (\ie $\E[\cdot]$ means $\E[\cdot \mid x_i, x_j]$). W.l.o.g., assume $x_i \leq x_j$. We use techniques similar to those developed in Theorem \ref{thm:wrong}. Specifically let
$p_i = \mathbbm{P}\left[y_1\succ_i y_2 \mid  y_1, y_2, x_i\right]$. We have
{\footnotesize
\begin{equation}\nonumber
\begin{split}
\mathbbm{E}_{R_i,R_j}\left[D(x_i, x_j)\right] &= \mathbbm{E}_{R_i,R_j}\left[\mathbbm{E}_{x_k}\Big|\mathbbm{E}_{\,y_1,\,y_2}\;\left\{\left[p_k(1-p_i)+p_i(1-p_k)\right]\right. - \left.\left[p_k(1-p_j)+p_j(1-p_k)\right]\right\}\Big|\right]\\
&= \mathbbm{E}_{R_i,R_j}\left[\mathbbm{E}_{x_k}\Big|\mathbbm{E}_{\,y_1\leq\,y_2}\left[\left(p_i-p_j\right)(1-2p_k)\right]\Big|\right].
\end{split}
\end{equation}
}
Let us define three events:
\begin{itemize}[noitemsep,topsep=0pt]
\setlength{\itemsep}{0pt}
\setlength{\parskip}{0pt}
\setlength{\parsep}{0pt}
\item $\cale_1$: when $y_1,y_2\in[0,x_i]$ or $y_1,y_2\in[x_j,c]$.
\item $\cale_2$: when $y_1\in[0,x_i]$ and $y_2\in[x_j,c]$.
\item $\cale_3$: when $y_1\in[x_i,x_j]$ or $y_2\in[x_i,x_j]$.
\end{itemize}

We compute $(p_i-p_j)(1-2p_k)$ conditioned on the  three events (\ie $\left[\left(p_i-p_j\right)(1-2p_k)\mid \cale_{\ell}\right]$ for $\ell = 1,2,3$).

\mypara{Event $\cale_1$:} One can see that $
\mathbbm{E}\left[\left(p_i-p_j\right)(1-2p_k)\mid \cale_1\right] = 0$ by using a symmetric argument.

\mypara{Event $\cale_2$.}
We have
{\footnotesize
\begin{eqnarray*}\footnotesize
p_j \mid \cale_2 & = &\frac{e^{-(x_j-y_1)}}{e^{-(x_j-y_1)}+e^{-(y_2-x_j)}}
 =  \frac{e^{-(x_i-y_1)-\epsilon}}{e^{-(x_i-y_1)-\epsilon}+e^{-(y_2-x_i)+\epsilon}} 
 \geq  p_i-\frac{\epsilon}{2}+O(\epsilon^2).
\end{eqnarray*}
}
\mypara{Event $\cale_3$.} We have
\begin{equation}\nonumber\footnotesize
\begin{split}
|p_i-p_j| =\;& \bigg|\frac{e^{-||x_i-y_1||}}{e^{-||x_i-y_1||}+e^{-||y_2-x_i||}}-\frac{e^{-||x_j-y_1||}}{e^{-||x_j-y_1||}+e^{-||y_2-x_j||}}\bigg|\\
\leq\;& \left|\frac{\partial\left[\frac{e^{-||x_i-y_1||}}{e^{-||x_i-y_1||}+e^{-||x_i-y_2||}}\right]}{\partial x_i}\right|\epsilon +O(\epsilon^2)= O(\epsilon).
\end{split}
\end{equation}
Therefore,
\begin{equation}\label{eqn:e3}
\mathbbm{E}_{\,y_1\leq\,y_2}\left[\left(p_i-p_j\right)(1-2p_k)\mid \cale_3\right]\leq\frac{\epsilon}{2}+ O(\epsilon^2).
\end{equation}

$D(x_i,x_j)\leq c_2'\epsilon'$ follows from combining all results for $\{\cale_i\}_i$.
\subsection{Lower bounds}
\begin{proposition}\label{prop:lower}
Using the notations above, let $|x_i - x_j| = \epsilon$ and assume that  $\cald_X$ and $\cald_Y$ are near-uniform on $[0,c]$. Then we have 
\begin{equation}
\E_{R_i,R_j,X\setminus\{x_i,x_j\}}[D(x_i, x_j)] \geq c_0 \cdot \epsilon^2. 
\end{equation}
\end{proposition}
There is a gap of factor $\epsilon$ between the upper  and lower bounds. Closing the gap is an interesting open problem. 

\begin{proof}[Proof of Proposition~\ref{prop:lower}]
Let 
$\calf_{x_k}(x) \equiv  \E_Y[\nkt(R_{k}, R_x) \mid x_k]$. Also let $\delta = c_0 \epsilon$ for some suitably small constant $c_0$.  

\begin{eqnarray*}
\E[D(x_i,x_j)]
& = & \E_X\left[\left|\E_Y[F_{ik}]-\E_Y[F_{jk}]\right|\mid x_k\in [0,\delta]\right]\cdot
\mathbbm{P}[x_k\in [0,\delta]]\\
& & + \E_X\left[\left|\E_Y[F_{ik}]-\E_Y[F_{jk}]\right|\mid x_k\in [\delta,c - \delta]\right]\cdot \mathbbm{P}[x_k\in [\delta,c-\delta]]\\
& & + \E_X\left[\left|\E_Y[F_{ik}]-\E_Y[F_{jk}]\right|\mid x_k\in [c-\delta, c]\right]\cdot \mathbbm{P}[x_k\in [c-\delta,c]]\\
& \geq & \E_X\left[\left|\E_Y[F_{ik}]-\E_Y[F_{jk}]\right|\mid x_k\in [0,\delta]\right]\cdot\mathbbm{P}[x_k\in [0,\delta]]+\\
& & \E_X\left[\left|\E_Y[F_{ik}]-\E_Y[F_{jk}]\right|\mid x_k\in [c-\delta,c]\right]\cdot \mathbbm{P}[x_k\in [c-\delta,c]]\\
\end{eqnarray*}
Using the near-uniform assumption, we have $\Pr[x_i \in [0, \delta]]$ and $\Pr[x_i \in [c-\delta, c]]$ are in the order of $\Theta(\epsilon)$. Furthermore, we have the following lemma. 

\begin{lemma} Using the notations above, there exists a constant $c_1$ (that's independent of $c_0$ and $c$) such that: 
\begin{itemize}[noitemsep,topsep=0pt]
\setlength{\itemsep}{0pt}
\setlength{\parskip}{0pt}
\setlength{\parsep}{0pt}
\item When $x_k \in [0, \delta]$, $|\calf_{x_k}(x_1) - \calf_{x_k}(x_2)| \geq |\calf_0(x_1) + \calf_0(x_2)| - c_1 \epsilon^2
$
\item When $x_k \in [c-\delta, c]$, $
|\calf_{x_k}(x_1) - \calf_{x_k}(x_2)| \geq |\calf_c(x_1) + \calf_c(x_2)| - c_1 \epsilon^2.
$
\end{itemize}
\end{lemma}
Thus, we have 
$$ \E[D(x_i, x_j)] \geq  c_2 \left[|\calf_0(x_1) - \calf_0(x_2)| + |\calf_c(x_1) - \calf_c(x_2)| - 2c_1 \epsilon^2\right] $$
Below is our key lemma in operating $\calf_0$ and $\calf_1$: 
\begin{lemma}\label{lem:calf} Using the notations above, we have 
$$| \calf_0(x_i) - \calf_0(x_j)| \geq c_3 |x_j-x_i|\cdot |3x_j(1-x_j) + 3x_i(1-x_i) + (x_i+x_j)^2|$$
and 
$$| \calf_c(x_i) - \calf_c(x_j)| \geq c_3 |x_j - x_i|\cdot|3x_j(1-x_j) + 3x_i(1-x_i) + (2-x_i-x_j)^2|.$$
\end{lemma}
One can see that we may use the above lemma to get $\E[D(x_i, x_j)] = \Omega(\epsilon^2)$. 

\begin{proof}[Proof of Lemma~\ref{lem:calf}] We shall prove only the first part of the lemma. The second part will be similar. W.l.o.g., assume $x_i < x_j$. We have
\begin{equation}\label{eqn:derivativetobound}
\left| \calf_0(x_{i}) - \calf_0(x_{j})\right| = \left|\int_{x_i}^{x_j}\frac{d\calf_0(x)}{dx} dx\right|.
\end{equation}
Thus, we shall focus on finding a lower bound on $\frac{d \calf_0(x)}{dx}$. We have
$$\frac{d\calf_0(x)}{dx} = \E_{y_1, y_2}\left[(1-2p_0)\frac{dp}{dx}\right].$$
where $p_0 = \mathbbm{P}[y_1\succ_{k}y_2 \mid x = 0 ]$
Let $\deltaz = y_2 - y_1$ and $\deltax = |y_2 - x| - |y_1 - x|$. We have
\begin{eqnarray*}
\E_{y_1, y_2}\left[(1-2p_0)\frac{dp}{dx}\right]
& = & \E_{y_1, y_2}\left[\frac{e^{-\deltaz} - 1}{e^{-\deltaz} + 1}\cdot \frac{\sign(y_1 - x) - \sign(y_2 - x)}{4\cosh^2(\deltax/2) }\right] \\
& = & \E_{y_1 < y_2}\left[\frac{e^{-\deltaz} - 1}{e^{-\deltaz} + 1}\cdot \frac{\sign(y_1 - x) - \sign(y_2 - x)}{4\cosh^2(\deltax/2) }\right]
\end{eqnarray*}

To simplify the notation, let $\Phi_0(y_1, y_2, x_i) = \frac{e^{-\deltaz} - 1}{e^{-\deltaz} + 1}\cdot \frac{\sign(y_1 - x) - \sign(y_2 - x)}{4\cosh^2(\deltax/2) }$ (sometimes we simply use $\Phi_0$ when the context is clear). We consider two cases. 

\mypara{Case 1: $y_1,y_2\geq x$ or $y_1,y_2\leq x$ (referred to as event $\cale_1$).} 
In this case, we have $\sign(y_1 - x) - \sign(y_2 - x) = 0$. Therefore, $\E[\Phi_0 \mid \cale_1] = 0$. 

\mypara{Case 2: $y_2\geq x\geq y_1$ (referred to as event $\cale_2$).}
One can see that $\Phi_0 \mid \cale_2 > 0$ because 
$\sign(y_1 - x) - \sign(y_2 - x) = -2 < 0$. Next, define $\cale_{2, 1} = (y_1\leq\frac{x}{2}) \wedge (y_2\geq\frac{1+x}{2})$. Note that $\cale_{2,1} \subset \cale_2$, we have, 
\begin{eqnarray*}
 \E_{y_1\leq y_2}[\Phi_0 \mid \cale_2]
 & \geq &  \E_{y_1 \leq y_2}\left[\Phi_0 \mid \cale_{2,1}\right]\cdot\; \mathbbm{P}\left[\cale_{2,1}\mid \cale_2\right]. 
\end{eqnarray*}
When $\cale_{2, 1}$ occurs, $\deltaz \geq \frac{c}{2}$, $\deltax \leq c$  and we have,
\begin{eqnarray*}
\E_{y_1\leq y_2}[\Phi_0 \mid \cale_{2,1}]& = & \E_{y_1 \leq y_2}\left[\frac{e^{-\deltaz} - 1}{e^{-\deltaz} + 1}\cdot \frac{-2}{4\cosh^2(\deltax/2) }\mid \cale_{2,1}\right] \\
& \geq & \E_{y_1 \leq y_2}\left[\frac{e^{-\frac{c}{2}} - 1}{e^{-\frac{c}{2}} + 1}\cdot \frac{-2}{4\cosh^2(\frac{c}{2}) }\mid \cale_{2,1}\right] \\
& = &\frac{1-e^{-\frac{c}{2}}}{1+e^{-\frac{c}{2}} }\cdot \frac{1}{2\cosh^2(\frac{c}{2}) } = c_3(c) >0.
\end{eqnarray*}
Combining the above two cases, we have
$\frac{d\calf_0(x)}{dx} \geq  \frac{c_3(c)}{2\left(1+c_Y\right)^2}\;x(1-x)$. 
 
Together with (\ref{eqn:derivativetobound}), we complete the proof of the first part in Lemma~\ref{lem:calf}.  

\end{proof}
Lemma~\ref{prop:lower} follows by finding applying $\delta = c_o\epsilon$ for some suitable constant $c_o$ 
\end{proof}


\section{Missing proofs for finding similar alternatives}\label{a:item}
This section presents missing proofs in Section~\ref{sec:item}.

\subsection{Proof for Lemma~\ref{lem:itemsim}}
We prove the upper bound first.

We consider two cases.

\mypara{Case 1: both $y_i$ and $y_j$ are at the same side of $\frac 1 2$.} We have $y_i, y_j \leq \frac 1 2$ or $y_i, y_j \geq \frac 1 2$. Wlog, assume that $y_i \leq y_j \leq \frac 1 2$. Let $r = 1 - y_i - y_j$. We have
{\small
\begin{eqnarray*}
\mu_x(y_i, y_j) & = &  \int_{x \in [0, r]}2I(R_x(y_i) - R_x(y_j) < 0) - 1\;dx   + \int_{x \in [r, 1]}2I(R_x(y_i) - R_x(y_j) < 0) - 1\;dx.
\end{eqnarray*}
}
We can use a symmetric argument  to show that $\int_{x \in [0, r]}2I(R_x(y_i) - R_x(y_j) < 0) - 1 dx = 0$. Therefore,
$\mu_x(y_i, y_j) = \int_{x \in [r, 1]}2I(R_x(y_i) - R_x(y_j) < 0) - 1 dx.$

Let $b(z) = \frac{1}{\exp(-z) + 1}$. The function $b(\cdot)$ is used to measure the probability that $y_i \prec_x y_j$ when $y_i \leq y_j \leq x$. We have
\begin{equation}\label{eqn:sameside}
2I(R_x(y_i) - R_x(y_j) < 0) - 1  = 2b(\delta_{i, j}) - 1 = \frac{\delta}{2} + o(\delta).
\end{equation}
The last equation uses a Taylor expansion.
Therefore,
{\footnotesize
$$\int_{x \in [r, 1]}2I(R_x(y_i) - R_x(y_j) < 0) - 1\;dx = (1-r)(\frac{\delta}{2} + o(\delta)) \leq c_2 \delta$$
}
for some constant $c_2$.

\mypara{Case 2: $y_i$ and $y_j$ are on different sides of $\frac 1 2$. } Wlog, assume that $y_i \leq \frac 1 2 \leq y_j$ and $\tau_i \geq \tau_j$ (other cases can be analyzed in a similar fashion). We can again use the same symmetric trick to show that
\begin{equation}
\mu_x(y_i, y_j) \leq (2b(|y_i - y_j|) - 1)(1 - y_i - y_j).
\end{equation}
Observe that $2b(|y_i - y_j|) - 1 \leq 2b(1) - 1 = \Theta(1)$. Furthermore, one can see that $1 - y_i - y_j \leq \delta_{i, j}$. Therefore, $\mu_x(y_i, y_j) = O(\delta_{i, j})$.

For the lower bound, we consider two cases.

\mypara{Case 1: $y_i$ and $y_j$ are at the same side of $\frac 1 2$.} Wlog, assume that $y_i \leq y_j \leq \frac 1 2$. Recall that $r = y_i + y_j$, and from (\ref{eqn:sameside}), we have 
$$\mu_x(y_i, y_j) = (2b(\delta_{i, j}) - 1) (1 - r).$$ 
We observe that $(1-r) \geq \delta_{i, j}$ and $(2b(\delta_{i, j}) - 1) = \Omega(\delta_{i, j})$ (by using a Taylor expansion again). We have $\mu_x(y_i, y_j) \geq c_1 \delta^2_{i, j}$ for some constant $c_1$. 

\mypara{Case 2: $y_i$ and $y_j$ are on different sides of $\frac 1 2$. } Wlog, we continue to assume that $y_i \leq \frac 1 2 \leq y_j$ and $\tau_i \leq \tau_j$. We again have $\mu_x(y_i, y_j) = (2b(\delta_{i, j}) - 1) (1 - r).$ One can see that even $y_i$ and $y_j$ are at two sides of $\frac 1 2$, $|y_i - y_j| \geq \delta_{i, j}$. Therefore, we still have $\mu_x(y_i, y_j) \geq c_1 \delta^2_{i, j}$ for some constant $c_1$.  

\subsection{Proof of Lemma~\ref{lem:cluster}}
Proving Lemma~\ref{lem:cluster} requires two steps: 
\begin{itemize}
\item The gap between $Q_1$ and $Q_2$ is sufficiently large so we need only polynomial number of agents (\ie $n$ is polynomial) to detect the gap. 
\item Even without the knowledge of $Q_1$ and $Q_2$, we can split the set $\bfC(y_i)$. 
\end{itemize}

\myparab{Gap between $Q_1$ and $Q_2$.} We now analyze the gap $Q_1 - Q_2$ to determine the sample complexity (\ie size of $n$). Wlog, assume that $y_i \leq \frac 1 2$. Let $q = \frac 1 2 - \theta$ for some $\theta$. We observe that $Q_1 - Q_2 = \Theta(\theta)$. Therefore, the sample complexity is $O(\theta^2 \log m)$. 

We then determine the size of $\theta$. We note that $q = \E_x[\cale_1(x, y_i) \mid y_i]$ and $\E_x[\cale_1(x, y_i) \mid y_i]$ is monotonically decreasing in $x$. When $y_i = 0$, $\theta = \Theta(1)$. In this case, the sample complexity is $O(\log m)$. When $y_i \rightarrow \frac 1 2$, $\theta \rightarrow 0$ so the sample complexity grows to infinite. But when $y_i \rightarrow \frac 1 2$, all alternatives in $\bfC(i)$ are neighbors of $y_i$ so we do not need to filter out any elements. 

The most ``harsh'' case to handle is when $y_i = \frac 1 2 - 2\delta$. In this case, while $y_i$ is close to the mid-point $\frac 1 2$, we still need to filter out alternatives that are around $1 - y_i$. Our crucial observation is that $\E_x[\cale_1(x, y_i) \mid y_i]$ is smooth and continuous on $y$ and $|d\E_x[\cale_1(x, y_i) \mid y_i]/dy| = \Theta(1)$. Therefore, we have $\theta \geq c_0 \delta$ for some $c_0$. This implies the sample complexity is $n = O(\frac{1}{\delta^2} \log n)$. With the knowledge of $Q_1$ and $Q_2$, we can use a simple rule to determine whether $y_j$ is a neighbor of $y_i$: if $S(y_i, y_j) > \frac{Q_1 + Q_2}{2}$, then $y_j$ is a neighbor of $y_i$. Otherwise, $y_j$ is a neighbor of $1-y_i$.

\myparab{The knowledge of $Q_1$ and $Q_2$.} We do not know the values of $Q_1$ and $Q_2$ so the above rule cannot be directly implemented. We next describe a simple trick to circumvent the issue. Specifically, let $\bfC^+ \subseteq \bfC$ be the set of $y_j$'s that are close to $y_i$ and $\bfC^- \subseteq \bfC$ be the set of $y_j$'s that are close to $1-y_i$. Define $\calq = \{S(y_i, y_j): j \in \bfC(y_i)\}$, $\calq^+ = \{S(y_i, y_j): j \in \bfC^+(y_i)\}$, and $\calq^- = \{S(y_i, y_j): j \in \bfC^-(y_i)\}$. We note that points in $\calq^+$ have center $Q_1$ and $|z - Q_1| = o(\delta)$ for any $z \in \calq^+$ (when $n = \omega(\frac 1 {\delta^2}\log n)$); similarly, points in $\calq^-$ have center $Q_2$ and $|z - Q_2| = o(\delta)$. Finally, $|Q_1 - Q_2| = \Omega(\delta)$. Therefore, using a standard $k$-means clustering algorithm for 1-dimensional lines, we can recover $\calq^+$ and $\calq^-$ from $\calq$ without the knowledge of $Q_1$ and $Q_2$. 

\section{Analysis of Algorithm~\ref{alg:split}}
We make two remarks here: \emph{(i)} We parameterize the error in terms of $\delta$ instead of $\ell$ because we aim to obtain a result that's independent to how $\bfC(y_i)$ is constructed; \emph{(ii)} We aim to find all $y_j$ such that $|y_j - y_i| \leq 5\delta$ (instead of $|y_j - y_i| \leq \delta$). This means when $y_i$ and $1 - y_i$ are too close (say $|y_i - (1-y_i)| \leq 2\delta$), we do not need to filter out any element in $\bfC(y_i)$.

We construct a statistics to determine whether $y_j$ is close to $y_i$ or $1 - y_i$. Let us first define two events:
\begin{itemize}[noitemsep,topsep=0pt]
\setlength{\itemsep}{0pt}
\setlength{\parskip}{0pt}
\setlength{\parsep}{0pt}
\item $\cale_1(x, y)$: the alternative $y$ is in the first half of $R_x$ (\ie $R_x(y) \leq \frac m 2$).
\item $\cale_2(x, y)$: the alternative $y$ is in the second half of $R_x$.
\end{itemize}
Next, define
\begin{equation}
S(y_i, y_j) = \frac 1 n\sum_{k \in [n]}\left(I\big(\bigvee_{t \in \{1, 2\}}(\cale_t(x_k, y_i) \wedge \cale_t(x_k, y_j)\big)\right)
\end{equation}
$I\big(\bigvee_{t \in \{1, 2\}}(\cale_t(x_k, y_i) \wedge \cale_t(x_k, y_j)\big)$ is $1$ if and only if both $y_i$ and $y_j$ are in the first half or second half of $R_x$.

We next  show that $\{S(y_i, y_j)\}_{j \in [m]}$ k-means cluster around two values, depending on whether $y_j$ is close to $y_i$ or $1-y_i$. For exposition purpose, we assume $y_j = y_i$ or $y_j = 1 - y_i$ and examine the behavior of $S(y_i, y_j)$. We can use standard Taylor expansion analysis to handle the case where we have only $|y_j - y_i| \leq \delta$ or $|y_j - (1-y_i)| \leq \delta$.

Note that $I\big(\bigvee_{t \in \{1, 2\}}(\cale_t(x_k, y_i) \wedge \cale_t(x_k, y_j)\big)$ are i.i.d. random variables for different $k$. Let $q$ be $\Pr[\cale_1(x_k, y_i) \mid y_i]$. We have $1 - q = \Pr[\cale_2(x_k, y_i) \mid y_i]$. When $y_j = y_i$, we have
\begin{equation}
Q_1 \equiv \E I\big(\bigvee_{t \in \{1, 2\}}(\cale_t(x_k, y_i) \wedge \cale_t(x_k, y_j)\big) = q^2 + (1-q)^2.
\end{equation}
This uses the fact that $\cale_t(x_k, y_i)$ and $\cale_t(x_k, y_i)$ are close to independent when $m$ is large.

When $y_j = 1-y_i$, we have
\begin{equation}
Q_2 \equiv \E I\big(\bigvee_{t \in \{1, 2\}}(\cale_t(x_k, y_i) \wedge \cale_t(x_k, y_j)\big) = 2q(1-q).
\end{equation}
Therefore, $Q_1 > Q_2$ when $q$ is bounded away from $\frac 1 2$. This also implies that when $n$ is sufficiently large, we have $S(y_i, y_j) \rightarrow Q_1$ when $y_j = y_i$ and $S(y_i, y_j) \rightarrow Q_2$ when $y_j = 1-y_i$. By running a simple clustering algorithm (over 1-dim space with performance guarantee, we can identify $y_j$'s that are close to $1-y_i$. See Algorithm~\ref{alg:split} and Lemma~\ref{lem:cluster}.
{\footnotesize}

\end{document}